\documentclass{article}

\usepackage[preprint]{neurips_2026}

\usepackage[utf8]{inputenc} % allow utf-8 input
\usepackage[T1]{fontenc}    % use 8-bit T1 fonts
\usepackage{hyperref}       % hyperlinks
\usepackage{url}            % simple URL typesetting
\usepackage{booktabs}       % professional-quality tables
\usepackage{amsfonts}       % blackboard math symbols
\usepackage{nicefrac}       % compact symbols for 1/2, etc.
\usepackage{microtype}      % microtypography
\usepackage{xcolor}     
\usepackage{graphicx}
\usepackage{subcaption}
\usepackage{booktabs}
\usepackage{multirow}
\usepackage{hyperref}
\usepackage{booktabs}
\usepackage{amsmath}
\usepackage{amssymb}
\usepackage{mathtools}
\usepackage{amsthm}% colors
\usepackage{microtype}
\usepackage{xcolor}

\usepackage{algorithm}
\usepackage{algorithmic}

\theoremstyle{plain}
\newtheorem{theorem}{Theorem}[section]

\newtheorem{lemma}[theorem]{Lemma}
\newtheorem{corollary}[theorem]{Corollary}
\theoremstyle{definition}

\theoremstyle{remark}

\usepackage{float}

\title{Posterior Continuation with Noise-Conditioned Frequency Exposure for Diffusion Inverse Problems}

\author{%
{\small\bfseries
Feng Tian\textsuperscript{1},\,
Yixuan Li\textsuperscript{1},\,
Weili Zeng\textsuperscript{1},\,
Weitian Zhang\textsuperscript{1},\,
Yichao Yan\textsuperscript{1,*},\,
Xiaokang Yang\textsuperscript{1}
}\\[0.25em]
{\normalfont
\textsuperscript{1}Shanghai Jiao Tong University
}
}

\begin{document}

\maketitle
\begingroup
\renewcommand{\thefootnote}{\fnsymbol{footnote}}
\footnotetext[1]{Corresponding authors: yanyichao@sjtu.edu.cn}
\endgroup
\begin{abstract}
Diffusion posterior sampling solves inverse problems by combining a pretrained diffusion prior with measurement-consistency guidance. However, full-band guidance can be unreliable at high noise levels, where clean estimates contain
score-induced errors and high-frequency measurement directions are weakly
identifiable. We argue that posterior guidance should expose measurement
frequencies according to the instantaneous diffusion noise level. Based on this
principle, we propose a posterior continuation framework that constructs a
family of intermediate posteriors whose likelihood emphasizes currently reliable
frequency bands and gradually returns to full-band consistency. We instantiate
this framework with a stabilized sampler that combines a diffusion predictor,
frequency-limited likelihood refinement, and a Haar-domain commitment rule that
commits reliable coarse corrections while deferring weakly identifiable details.
Across super-resolution, inpainting, and deblurring, our method achieves
competitive-to-state-of-the-art restoration performance, including up to 5 dB
PSNR improvement on motion deblurring over strong baselines  in evaluations on
FFHQ and ImageNet.

\end{abstract}

\section{Introduction}
Diffusion models \cite{ddpm, ddim, scoresde, cm} have emerged as powerful priors for solving challenging inverse problems \cite{ddrm, ddnm} without task-specific training, by combining pretrained diffusion priors with measurement-consistency guidance during sampling. Despite their empirical success, training-free diffusion posterior sampling methods often struggle to recover fine details under severe degradations such as strong blur, heavy downsampling, or ill-conditioned operators. These failures persist even when the forward model is known exactly, suggesting a limitation that is not merely algorithmic, but structural. A key observation underlying these failures is that what is reliably recoverable during sampling depends jointly on the diffusion noise level and the signal frequency \cite{freq_3, freq}. As shown in Fig.~\ref{method_1}, in early and intermediate stages, the iterate remains heavily corrupted by noise, and high-frequency components exhibit intrinsically low signal-to-noise ratio, making them weakly identifiable. In contrast, low-frequency structure tends to be more stable and often preserves global geometry. Nevertheless, most existing posterior samplers inject measurement information
uniformly across frequencies and noise levels, without explicitly considering
their noise-dependent reliability.
\begin{figure}[t]
\centering
\includegraphics[width=\linewidth]{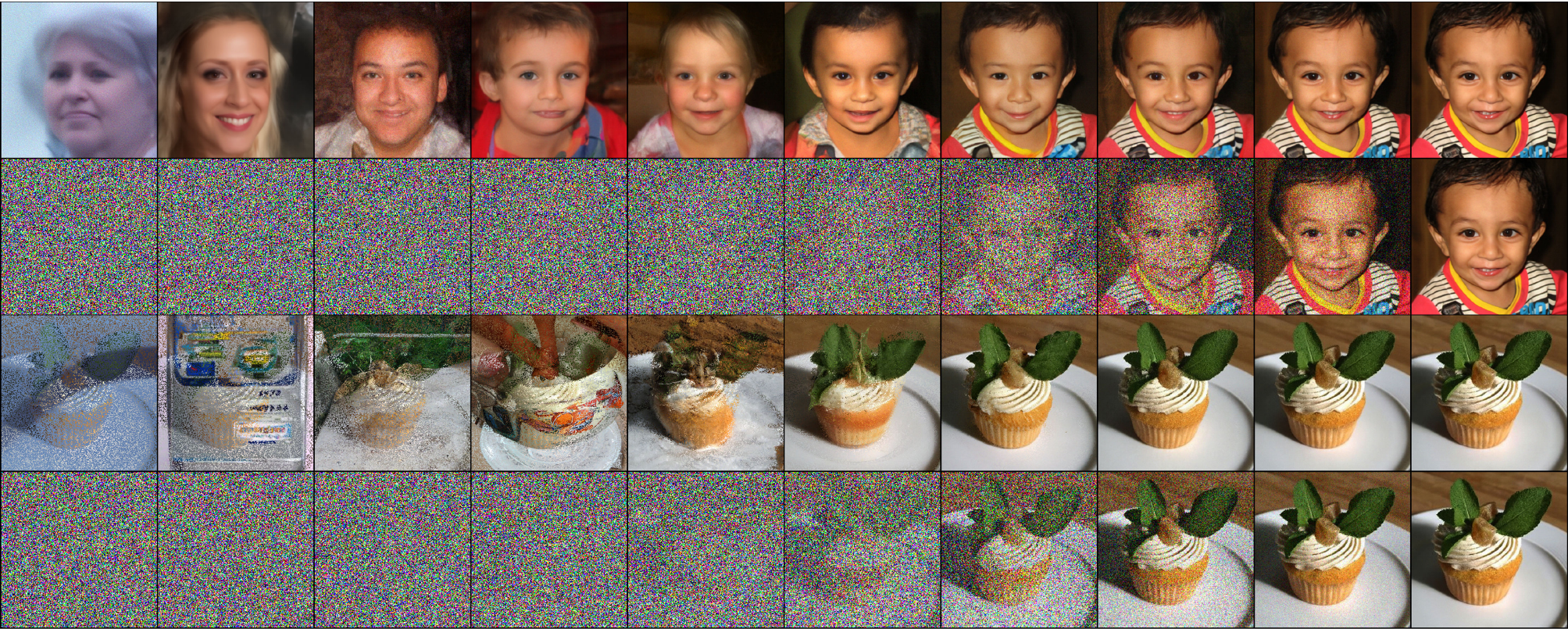}
\caption{\textbf{Noise-dependent recoverability along the sampling trajectory.}
The first and third rows show estimated clean images on FFHQ and ImageNet, respectively, while the second and fourth rows show the corresponding noisy samples at different timesteps. Coarse layouts remain relatively stable throughout sampling, whereas fine details become reliable only as the noise level decreases.}
\label{method_1}
\vspace{-6mm}
\end{figure}

Early influential works reflect related intuitions from complementary angles, including ILVR \cite{ilvr} for low-frequency conditioning, DDRM \cite{ddrm} for efficient linear inverse sampling via operator spectral factorization, DDNM \cite{ddnm} for zero-shot restoration through a null-space formulation, and DPS \cite{dps} for general inverse problems using approximate posterior score guidance. However, under severe corruption, these samplers exhibit recurring instabilities. Fine details are often under-recovered or distorted, producing ringing or texture wash-out when constraints are applied aggressively or without spectral awareness. Moreover, at high noise levels, likelihood gradients computed from inaccurate estimates can be misaligned with the local posterior geometry, causing early structural drift that is hard to undo. Finally, the interaction between noise level and constraint strength is usually tuned heuristically, making performance sensitive to schedules, step sizes, and operator conditioning, with the largest degradation on ill-conditioned operators, such as strong motion deblurring and super-resolution.

A rich line of work has attempted to mitigate these issues by modifying either the posterior correction mechanism or the sampling trajectory. Optimization-based methods strengthen proximal interpretations of sampling, such as DiffPIR \cite{diffpir} and DPnP \cite{dpnp}. Particle-based methods reduce approximation error and better represent multi-modality, such as FPS \cite{fps} and DCDP \cite{dcdp}. PostEdit \cite{postedit} directly optimizes the estimated clean data to enable larger global corrections. Most relevant to stability and detail recovery, frequency-aware designs inject spectral structure into likelihood guidance. Frequency-domain moving-average sampling \cite{freq_1} damps unstable
high-frequency fluctuations, while FGPS \cite{freq} progressively incorporates
frequencies along a sampling schedule. While these approaches substantially improve robustness, they often treat noise scheduling and frequency shaping as separable design choices, rather than as two aspects of a single posterior geometry. In particular, existing frequency curricula are commonly parameterized by discrete sampling steps or solver-specific time indices. As a result, their semantics can shift when changing the number of function evaluations or switching solvers, complicating tuning and obscuring a principled understanding of stability. More fundamentally, these methods lack an explicit posterior-level prescription for how measurement information should be injected as a function of the instantaneous noise level.

In this work, we propose a posterior continuation framework based on noise-conditioned frequency exposure. Our central principle is that measurement frequencies should be exposed according to the instantaneous diffusion noise level. Instead of enforcing full-band measurement consistency uniformly throughout sampling, we construct a family of intermediate posteriors whose likelihood emphasizes frequency bands that are reliable at the current noise level and gradually returns to full-band consistency as the noise decreases. This gives frequency exposure a solver-independent meaning tied to the noise scale, rather than to an arbitrary step index. We instantiate this principle with an efficient training-free sampler. At each noise level, a diffusion predictor first estimates a clean image, and a frequency-limited likelihood refinement corrects it using the currently exposed measurement bands. Because early band-limited likelihoods leave high-frequency detail directions weakly constrained, naive refinement can introduce operator-dependent artifacts. We therefore introduce a post-refinement Haar-domain commitment rule that commits reliable coarse corrections immediately while deferring weakly identifiable details until their corresponding measurement frequencies are exposed. Across super-resolution, inpainting, and deblurring, our method achieves competitive-to-state-of-the-art restoration performance, with particularly large gains on ill-conditioned operators such as motion blur. Our main contributions are summarized as follows.
\begin{itemize}

\item We introduce a posterior-continuation view of diffusion inverse problems,
where measurement bandwidth is coupled to the instantaneous diffusion noise
rather than to solver-specific step indices, to obtain a solver-independent
exposure schedule.

\item We provide a theoretical characterization of high-noise likelihood
mismatch, showing how score-induced clean-estimate errors enter measurement
gradients and how frequency truncation removes unexposed high-frequency
mismatch.

\item We design a practical sampler with noise-conditioned frequency exposure
and a post-refinement Haar commitment rule that commits reliable coarse
corrections while gating weakly identifiable details across the sampling
trajectory.

\item Across super-resolution, inpainting, and deblurring, our method achieves
competitive-to-state-of-the-art restoration performance, including up to 5 dB
PSNR improvement on motion deblurring in evaluations on FFHQ and ImageNet.

%\item We introduce a Noise--Frequency Continuation view of diffusion posterior sampling by constructing a continuous family of intermediate posteriors that couples noise level and measurement bandwidth, rather than enforcing full-band measurement guidance uniformly across the trajectory.
%\item We design a practical sampler that realizes this view via a frequency-controlled likelihood with an explicit bandwidth schedule and a principled transition back to full measurement consistency, improving stability under ill-conditioned degradations.
%\item We propose a Haar-based multi-resolution consistency strategy that addresses the semantic information gap induced by band-limited likelihoods, enabling reliable coarse correction while preventing spurious high-frequency artifacts.
%\item Extensive experiments on super-resolution, inpainting, and deblurring demonstrate consistent improvements over prior training-free diffusion posterior samplers, including around \(5\) dB PSNR gains on motion deblurring while maintaining strong SSIM and LPIPS.
\end{itemize}

\section{Preliminary}\label{inverse_pre}
We consider an inverse problem that aims to sample the clean signal
\(\boldsymbol{x}_0\in\mathbb{R}^d\) from the posterior
\(p(\boldsymbol{x}_0\mid\boldsymbol{y})\), given a degraded measurement
\(\boldsymbol{y}\in\mathbb{R}^n\):
\begin{equation}\label{inv_define}
    \boldsymbol{y} =\mathcal{A}\left(\boldsymbol{x}_0\right) + \boldsymbol{n}, \quad \boldsymbol{n}\sim\mathcal{N}\left(\boldsymbol{0}, \sigma_y^2\boldsymbol{I}\right)
\end{equation}
Log density terms \(\nabla_{\boldsymbol{x}_t} \log p(\boldsymbol{x}_t)\) are
applied to solve the inverse problem, and we can derive
\(\nabla_{\boldsymbol{x}_t} \log p(\boldsymbol{x}_t|\boldsymbol{y})\) by
Bayes' rule under posterior score guidance
\begin{equation}\label{invsc}
    \nabla_{\boldsymbol{x}_t} \log p(\boldsymbol{x}_t|\boldsymbol{y}) = \nabla_{\boldsymbol{x}_t} \log p(\boldsymbol{x}_t) + \nabla_{\boldsymbol{x}_t}\log p( \boldsymbol{y}|\boldsymbol{x}_t),
\end{equation}
where the first term on the right-hand side is provided by the pre-trained diffusion model, whereas the second term is typically intractable. The measurement $\boldsymbol{y}$ induces an observation-dependent constraint that guides the sampling trajectory toward reconstructions consistent with the input, thereby complementing the prior $p(\boldsymbol{x})$ encoded by the diffusion model. To obtain an explicit and tractable surrogate for the second term, the existing
method DPS~\cite{dps} adopts the following approximation to enable tractable
guidance
\begin{equation}
\begin{aligned}
    p(\boldsymbol{y}|\boldsymbol{x}_t) = \mathbb{E}_{\boldsymbol{x}_0 \sim p(\boldsymbol{x}_0|\boldsymbol{x}_t)}[p(\boldsymbol{y}|\boldsymbol{x}_0)] 
    \approx p(\boldsymbol{y}|\hat{\boldsymbol{x}}_0), \ \ \hat{\boldsymbol{x}}_0=\mathbb{E}_{\boldsymbol{x}_0 \sim p(\boldsymbol{x}_0|\boldsymbol{x}_t)}[\boldsymbol{x}_0].
\end{aligned}
\end{equation}
The posterior mean \(\hat{\boldsymbol{x}}_0\) can be estimated by the pretrained
diffusion model, for example via Tweedie's formula under VP-SDE~\cite{scoresde}
or DDIM-style denoising~\cite{ddim}. This yields a tractable approximation to
\(p(\boldsymbol{x}_{t-1}\mid \boldsymbol{x}_{t},\boldsymbol{y})\) under the
posterior score guidance in Eq.~\ref{invsc}.

\section{Methods}\label{sec:method}
\subsection{Theoretical Analysis and Method Overview}
We use image restoration as a running example to motivate our method. 
Posterior sampling for restoration must recover both coarse structure and 
fine details, but the likelihood term in Eq.~\ref{invsc} is typically evaluated 
through an estimated clean image $\hat{\boldsymbol{x}}_0$ rather than the 
unknown clean signal. At high noise levels, this estimate can contain 
non-negligible score-induced errors. These errors do not imply unconditional 
monotone amplification for every score model and forward operator. Instead, 
they induce a noise-dependent sensitivity pathway. Measurement-driven 
corrections computed from an inaccurate $\hat{\boldsymbol{x}}_0$ can therefore 
be misaligned with the local posterior geometry. By contrast, low-frequency 
content remains comparatively stable under high noise and provides a reliable 
scaffold for preserving global layout, as illustrated in Fig.~\ref{method_1}.
 
We formalize this observation in Appendix~\ref{proof}. Theorem~\ref{gradient_mismatch} first gives a worst-case sensitivity bound for the likelihood-gradient error. Theorem~\ref{exact} then provides an exact spectral characterization: measured score-error components enter the likelihood-gradient mismatch through a \(\sigma^2 \mathcal{A}^{\top}\mathcal{A}\) factor. Consequently, under a non-degeneracy condition on the measured projection of
the score error, the corresponding projected mismatch component grows with the diffusion noise level. Corollary~\ref{truncation} further shows that frequency truncation removes the mismatch contributed by unexposed high-frequency bands. This directly motivates our noise-conditioned frequency exposure: at high noise,
the likelihood should primarily emphasize reliable frequency components, and as
the noise decreases, the passband is gradually expanded toward full-band
consistency. Finally, Theorem~\ref{thm:unstable} shows that such gradient errors can break descent
even under otherwise valid step sizes, while Corollary~\ref{sensitive} shows that band-limited
likelihoods reduce effective curvature and make the resulting updates less
schedule- and operator-sensitive during high-noise posterior refinement.

To address these issues, we introduce a posterior-continuation sampler with
noise-conditioned frequency exposure and a post-update multi-resolution
commitment rule. Let \(\{\tau_i\}_{i=0}^{N}\) denote the outer reverse-time grid and
\(\sigma_i=\sigma(\tau_i)\) the corresponding noise levels, with
\(\sigma_N>\cdots>\sigma_0=0\). We iterate from \(i=N\) to \(1\). Starting from \(x_{\tau_N}\sim\mathcal N(\boldsymbol{0},\sigma_N^2\boldsymbol{I})\), we
first use a pretrained diffusion solver to obtain a clean estimate
\(\hat{\boldsymbol{x}}_{0,i}\) at noise level \(\sigma_i\). We then refine this estimate under a noise-conditioned likelihood that
emphasizes reliable measurement frequencies at the current noise level and
gradually transitions to full-band consistency. The cutoff
\(\omega_i\) expands as \(\sigma_i\) decreases, so the sampler moves from coarse
measurement correction to full-band consistency. Finally, because early band-limited likelihoods leave detail coefficients weakly
constrained, we apply a Haar-domain commitment step: coarse corrections are
committed immediately, whereas refined details are gated until the corresponding
measurement frequencies become active. The procedure is summarized in Alg.~\ref{alg:main} and Fig.~\ref{main}, which
illustrate the algorithmic flow and module interactions.

In the following two sections, Sec.~\ref{log_density} details the noise-conditioned frequency-exposure
likelihood used to refine \(\hat{\boldsymbol{x}}_{0,i}\), and Sec.~\ref{sec_har} describes the Haar-domain commitment rule.
\begin{figure}[t]
    \centering
    \includegraphics[width=\linewidth]{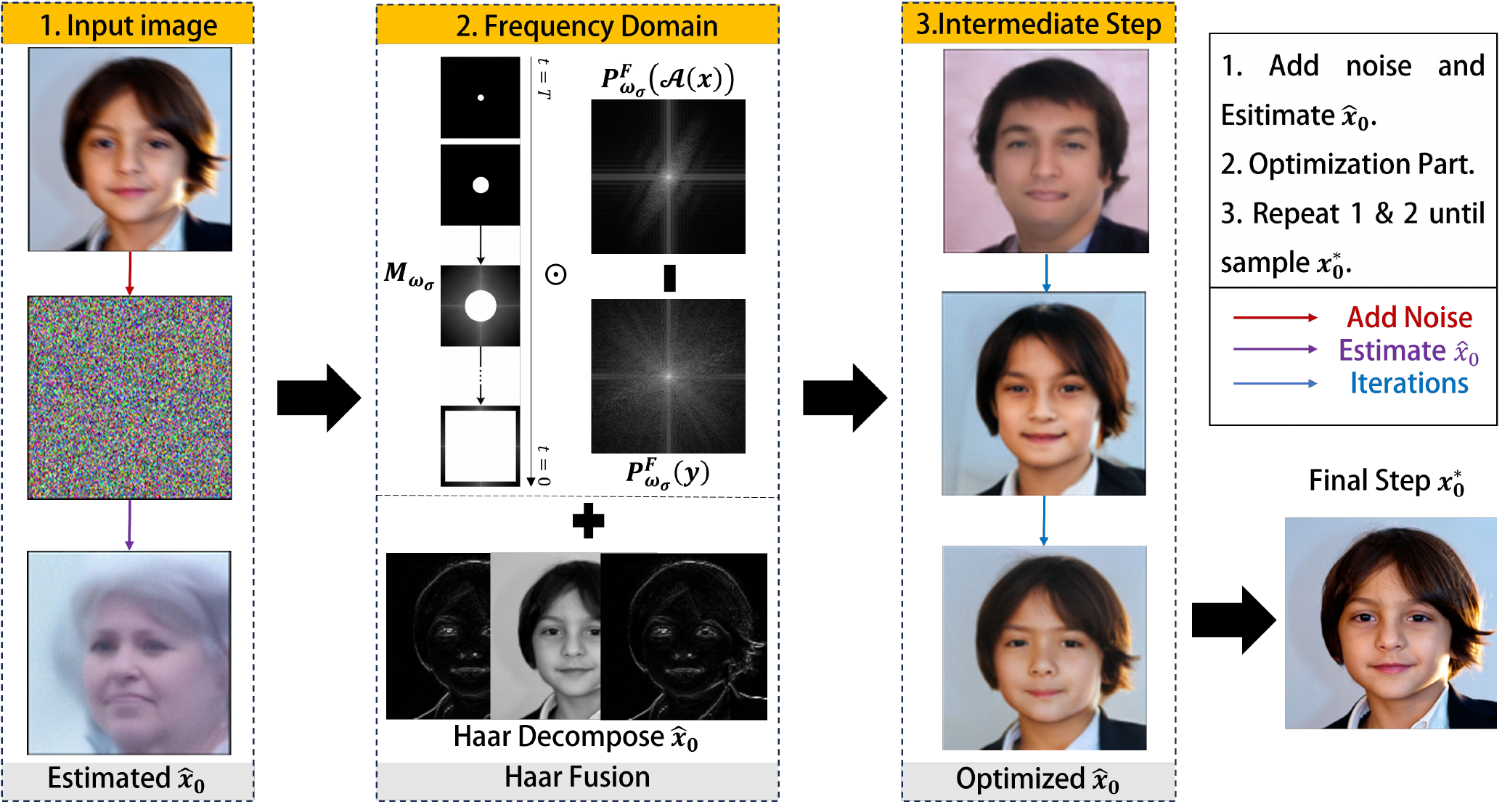}
    \vspace{-15pt}
    \caption{\textbf{Method overview.} At each noise level, a pretrained diffusion
    solver first predicts a clean estimate \(\hat{\boldsymbol{x}}_{0,i}\). We refine
    it with a noise-conditioned frequency-exposure likelihood that emphasizes
    reliable bands and gradually returns to full-band consistency. A Haar-domain commitment step then commits coarse corrections while gating
    details until they become identifiable.}
    \label{main}
    \vspace{-5mm}
\end{figure}

\subsection{Noise-Conditioned Frequency Exposure as Likelihood Design}
\label{log_density}
We now define the noise-conditioned likelihood used to refine
\(\hat{\boldsymbol{x}}_{0,i}\). Given an observation \(\boldsymbol{y}\) and a
forward operator \(\mathcal{A}\), the goal is to emphasize measurement
consistency on frequency components that are reliable at the current noise level,
while gradually returning to full-band consistency as the noise decreases. To
selectively enforce measurement consistency within a controllable frequency
band, we define the (normalized) 2D discrete Fourier transform for a
multi-channel tensor \(\boldsymbol{z}\in\mathbb{R}^{C\times H\times W}\) by
\begin{equation}
\begin{aligned}
    \mathcal{F}(\boldsymbol{z})_{c,u,v}
=
\frac{1}{\sqrt{HW}}
&\sum_{h=0}^{H-1}
\sum_{w=0}^{W-1}
\boldsymbol{z}_{c,h,w}\exp\!\left(
-2\pi i\left(\frac{uh}{H}+\frac{vw}{W}\right)
\right).
\end{aligned}
\end{equation}
Let \(\mathcal{J}\) denote the frequency-shift operator that centers the
zero-frequency component. We define coordinates \((u,v)\) on the shifted
frequency grid and let \(r_{\mathrm{Nyq}}\) be the Nyquist radius of the
observable residual grid. For a normalized cutoff \(\omega\in[0,1]\), the radial
low-pass mask is
\begin{equation}
M_\omega(u,v)
=
\mathbf{1}\left(
\frac{\sqrt{u^2+v^2}}{r_{\mathrm{Nyq}}}<\omega
\right),
\end{equation}
with the mask broadcast across channels.
We then define the masked, shifted spectrum
\begin{equation}\label{eq:PF_def}
P^F_\omega(\boldsymbol{z})
:=
\mathcal{J}\big(\mathcal{F}(\boldsymbol{z})\big)
\odot
M_\omega .
\end{equation}
where \(\odot\) denotes elementwise multiplication. Throughout,
\(\|\cdot\|_2^2\) applied to complex spectra denotes the sum of squared
magnitudes, i.e., \(\|Z\|_2^2=\sum_{c,u,v}|Z_{c,u,v}|^2\). With this
normalization, Parseval's identity implies that
\(\|P^{F}_{\omega}(\boldsymbol{r})\|_2^2\) measures the band-limited residual energy.

The frequency mask is applied in the observable residual space. For operators
whose outputs remain image-like, such as blur and masking, \(P^F_{\omega_i}\)
acts directly on \(\mathcal{A}(\boldsymbol{x})-\boldsymbol{y}\). For
resolution-changing operators, such as super-resolution, the same construction
is applied on the corresponding measurement grid. The band-limited measurement
loss is
\begin{equation}
\mathcal{L}_{\mathrm{freq}}(\boldsymbol{x}, \boldsymbol{y}, \omega_i)
=
\left\|P^{F}_{\omega_i}\!\big(\mathcal{A}(\boldsymbol{x})-\boldsymbol{y}\big)\right\|_2^{2}.
\label{eq:Lfreq}
\end{equation}
The band-limited term measures measurement consistency within the passband
specified by \(\omega_i\). Let \(\sigma_i\) denote the current noise level on the sampling trajectory. We schedule a cutoff \(\omega_i=\omega(\sigma_i)\) and a mixing weight \(\lambda_i=\lambda(\sigma_i)\in[0,1]\). The frequency-guided measurement objective is
\begin{equation}\label{eq:fg_loss}
\begin{aligned}
\mathcal{L}_{i}(\boldsymbol{x};\boldsymbol{y})
:=
(1-\lambda_i)\,\|\mathcal{A}(\boldsymbol{x})-\boldsymbol{y}\|_2^2
+
\lambda_i\,\mathcal{L}_{\mathrm{freq}}(\boldsymbol{x},\boldsymbol{y},\omega_i).
\end{aligned}
\end{equation}
Eq.~\eqref{eq:fg_loss} defines a noise-dependent likelihood that interpolates 
between band-limited and full-band measurement consistency. Early in the trajectory, a positive \(\lambda_i\) and a small \(\omega_i\)
attenuate unexposed high-frequency residuals. As \(\sigma_i\) decreases,
\(\omega_i\) expands and \(\lambda_i\) decays, recovering full-band consistency
near the terminal stage. Corollary~\ref{cor:mixed_gradient_error} shows that this 
mixed objective directly attenuates the likelihood-gradient mismatch: outside 
the exposed passband, the score-induced error contribution is weighted by 
$1-\lambda_i$. This provides a posterior-level justification for using 
noise-conditioned frequency exposure rather than full-band guidance throughout 
sampling.

We then formulate an intermediate posterior density whose log-gradient will be
used to refine the estimated \(\hat{\boldsymbol{x}}_{0,i}\) from the PF-ODE at
noise level \(\sigma_i\). Specifically, using the frequency-guided measurement
objective \(\mathcal{L}_i(\cdot\,;\,\boldsymbol{y})\) in Eq.~\eqref{eq:fg_loss},
we define the intermediate target distribution that combines measurement
consistency with a Gaussian anchor centered at
\(\hat{\boldsymbol{x}}_{0,i}\)

\begin{equation}\label{had}
\pi_i(\boldsymbol{x}_0)
\ \propto\ 
\exp\!\left(
-\frac{1}{2\beta_i^{2}}\,
\mathcal{L}_{i}(\boldsymbol{x}_0;\boldsymbol{y})
\right)
\cdot 
\mathcal{N}\!\left(
\boldsymbol{x}_0;
\hat{\boldsymbol{x}}_{0,i},
\sigma_i^2 \boldsymbol{I}
\right),
\end{equation}
where \(\beta_i>0\) is a temperature coefficient controlling the strength of the measurement term. To maintain stochastic exploration under the intermediate posterior, we employ Langevin dynamics \cite{lgvd} and perform $T$ refinement steps (indexed by \(\ell\)) at each outer step \(i\).
Starting from \(\boldsymbol{x}_{0,i}^{(0)}=\boldsymbol{\hat{x}}_{0,i}\), the Langevin refinement update is written as
\begin{equation}\label{ula_update}
\begin{aligned}
\boldsymbol{x}_{0,i}^{(\ell+1)}
=
\boldsymbol{x}_{0,i}^{(\ell)}
+
\eta_{i,\ell}\,
\nabla_{\boldsymbol{x}_0}
\log \pi_i\!\left(\boldsymbol{x}_{0,i}^{(\ell)}\right)
+
\sqrt{2\eta_{i,\ell}}\,
\boldsymbol{\xi}_{i,\ell},
\end{aligned}
\end{equation}
where $\boldsymbol{\xi}_{i,\ell}\sim\mathcal{N}(\boldsymbol{0},\boldsymbol{I})$,
and we take
\(\boldsymbol{x}_{0,i}^{\mathrm{ref}}:=\boldsymbol{x}_{0,i}^{(T)}\)
as the refined estimate.
In practice, \(\eta_{i,\ell}\) can be reduced as \(\beta_i\) decreases to maintain stable updates.
All implementation details and hyperparameter settings are provided in Appendix~\ref{imple}.

\subsection{Post-update Haar-domain Detail Commitment} \label{sec_har} 
The frequency-exposure refinement in Sec.~\ref{log_density} is designed to reduce high-noise drift by restricting the likelihood to identifiable frequency bands. However,
early band-limited likelihoods also leave detail coefficients weakly constrained.
Under the idealized low-pass model in Lemma~\ref{lp_insensitive_detail}, the data
term has vanishing gradient along high-frequency subspaces, so the corresponding
posterior directions are dominated by the Gaussian anchor around
\(\hat{\boldsymbol{x}}_{0,i}\). When the noise level is large, the refinement
process in Eq.~\ref{ula_update} can therefore inject high-variance,
operator-dependent detail components, including null-space or near-null-space
artifacts. This motivates Haar fusion as an explicit commitment mechanism:
coarse corrections are committed aggressively, while refined details are adopted
only when they become identifiable under the evolving posterior geometry.

Corollary~\ref{detail_prior_dominated} further explains this effect. When the
likelihood provides little information about Haar-detail components, their
conditional posterior is dominated by the Gaussian anchor centered at
\(\hat{\boldsymbol{x}}_{0,i}\) after being mapped into the Haar domain, i.e.,
\[
\boldsymbol{z}_{\mathrm{d}} \mid
(\boldsymbol{y},\boldsymbol{z}_{\mathrm{c}})
\approx
\mathcal{N}(\hat{\boldsymbol{z}}_{\mathrm{d},i},
\sigma_i^2 \boldsymbol{I}).
\] At high noise levels, this posterior yields refined high-frequency coefficients
with high variance and pronounced operator dependence. Let
\(\hat{\boldsymbol{x}}_{0,i}\) denote the image estimated by PF-ODE and
\(\boldsymbol{x}^{\mathrm{ref}}_{0,i}\) the refined estimate produced by the
Langevin update in Eq.~\ref{ula_update}. We use an orthonormal Haar transform \(W\). The coarse block
\(\mathrm{c}\) denotes the low-pass Haar coefficients, and the detail block
\(\mathrm{d}\) collects the remaining high-pass coefficients across orientations
and levels. We apply \(W\) to both \(\hat{\boldsymbol{x}}_{0,i}\) and
\(\boldsymbol{x}^{\mathrm{ref}}_{0,i}\):
\begin{equation}\label{haar_decomp}
\hat{\boldsymbol{z}}_i = W\hat{\boldsymbol{x}}_{0,i},
\qquad
\boldsymbol{z}^{\mathrm{ref}}_i = W\boldsymbol{x}^{\mathrm{ref}}_{0,i}.
\end{equation}
A Haar-domain commitment step that fuses coefficients bandwise is performed as
\begin{equation}\label{haar_fusion}
\boldsymbol{z}^{\mathrm{fuse}}_{i,b}
=
(1-w_{i,b})\,\hat{\boldsymbol{z}}_{i,b}
+
w_{i,b}\,\boldsymbol{z}^{\mathrm{ref}}_{i,b},
\qquad
b\in\{\mathrm{c},\mathrm{d}\},
\end{equation}
followed by \(\boldsymbol{x}^{\mathrm{fuse}}_{0,i}=W^\top \boldsymbol{z}^{\mathrm{fuse}}_i\). Importantly, Haar fusion is a post-refinement commitment operator rather than
an additional likelihood term. The frequency-limited likelihood in
Eq.~\eqref{eq:fg_loss} determines the Langevin gradient in Eq.~\eqref{ula_update}
and produces \(\boldsymbol{x}^{\mathrm{ref}}_{0,i}\). Haar fusion is then applied only after
this refinement step to decide which refined components are committed. Thus, it does not alter the likelihood gradient; instead, it
prevents weakly constrained detail directions from being committed before the
corresponding measurement frequencies are exposed. We commit coarse corrections aggressively by setting \(w_{i,\mathrm{c}}=1\), while adopting refined details progressively by a gated schedule
\begin{equation}\label{eq:detail_gate}
w_{i,\mathrm{d}} = 
\underbrace{\Big(d_s+(d_e-d_s)\rho_i^{\gamma}\Big)}_{\text{unlock}(i)}
\cdot 
\underbrace{(1-\lambda_i)}_{\text{gated schedule}},
\qquad
\rho_i=\frac{N-i}{N-1}.
\end{equation}
Here, \((1-\lambda_i)\) uses the exposure schedule as a clock for detail
commitment. It plays different roles in the two modules. In
Eq.~\eqref{eq:fg_loss}, \(\lambda_i\) changes the likelihood gradient before
refinement. In Eq.~\eqref{eq:detail_gate}, it only gates the post-refinement commitment and introduces no
additional likelihood term. This synchronization
prevents details from being committed before the corresponding measurement
frequencies are exposed. The factor \(\text{unlock}(i)\) increases monotonically from \(d_s\) to \(d_e\),
with \(\gamma>0\) controlling its curvature. Since \(0\le d_s\le d_e\le1\)
and \(\lambda_i\in[0,1]\), we have \(w_{i,d}\in[0,1]\).

Finally, we obtain the next state by re-noising the fused estimate according to the noise schedule
\begin{equation}
\boldsymbol{x}_{\tau_{i-1}}
\sim 
\mathcal{N}\!\left(
\boldsymbol{x}^{\mathrm{fuse}}_{0,i},
\sigma_{i-1}^2 \boldsymbol{I}
\right).
\end{equation}

\section{Experiments}\label{sec:exp}
\begin{figure*}[!t]
    \centering
    \includegraphics[width=\linewidth]{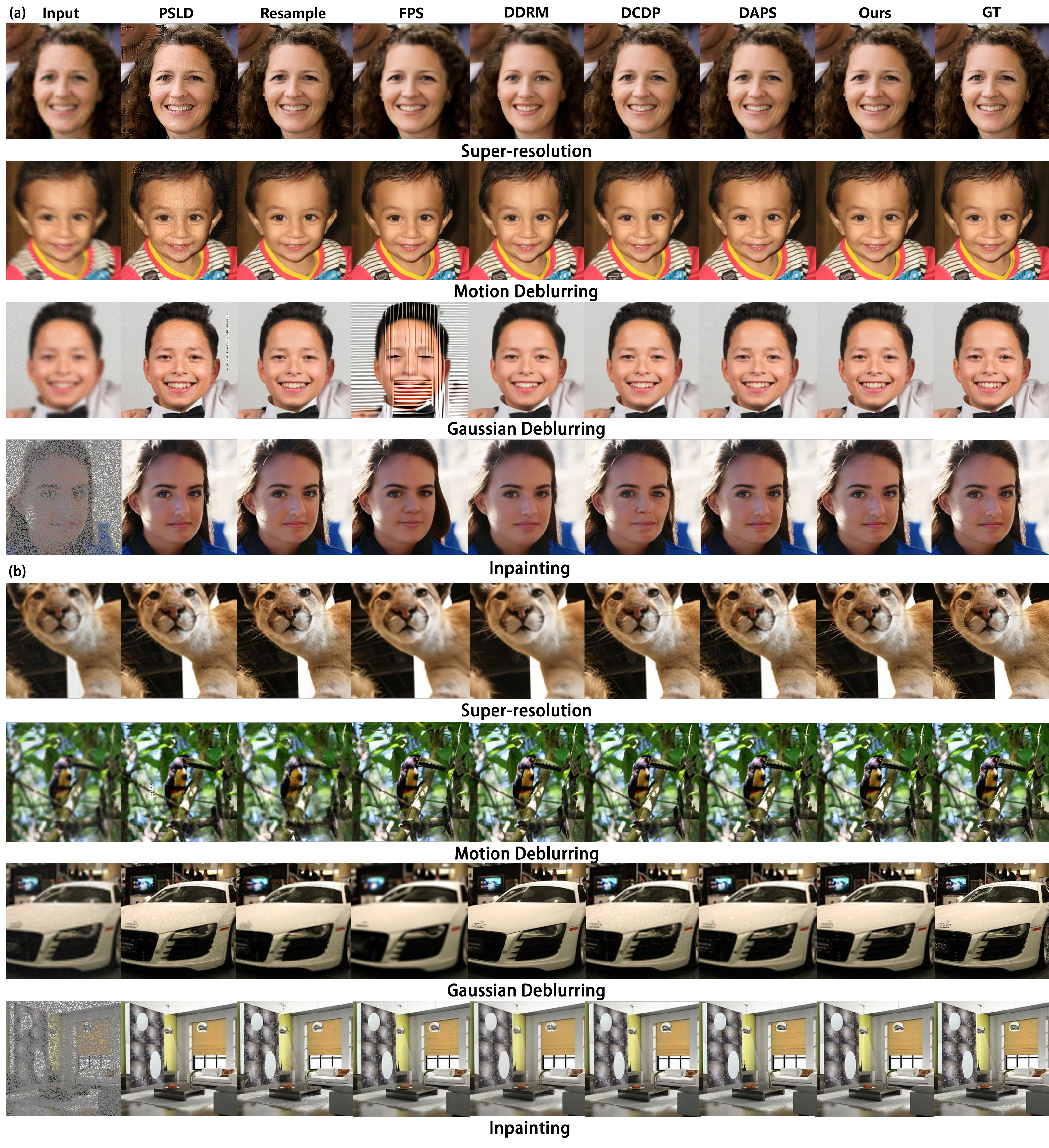}
    \vspace{-15pt}
    \caption{\textbf{Qualitative comparisons on FFHQ and ImageNet.}
The first four rows show FFHQ results, and the last four rows show ImageNet results. Our method better preserves structure and details while suppressing ringing and spurious high-frequency artifacts under challenging degradations.}
    \label{fig_main}
    \vspace{-7mm}
\end{figure*}

\subsection{Experimental Setup}\label{sec:set_up}
We use the pretrained FFHQ diffusion model from~\cite{dps} and the pretrained ImageNet diffusion model from~\cite{diffusion_beat}. We evaluate four representative linear inverse problems: \(4\times\) bicubic super-resolution, motion deblurring, Gaussian deblurring, and random inpainting. Following~\cite{dps,daps}, we synthesize degraded observations using \(61\times61\) blur kernels for deblurring and 70\% random masks for inpainting. All experiments target \(256\times256\) outputs and are evaluated by PSNR, SSIM, and LPIPS. We set the measurement noise level to \(\sigma_y=0.05\) in Eq.~\ref{inv_define} and, unless otherwise specified, evaluate on 100 validation images from each dataset. Additional implementation details and hyperparameters are provided in Appendix~\ref{imple}. 

We compare with representative baselines, including posterior-sampling methods such as DAPS~\cite{daps}, FGPS~\cite{freq}, FPS~\cite{fps}, DCDP~\cite{dcdp}, DPS~\cite{dps}, and ReSample~\cite{resample}, as well as diffusion-prior solvers such as DDNM~\cite{ddnm} and DDRM~\cite{ddrm}. We use official implementations whenever available and otherwise follow the settings reported in~\cite{daps}. Across methods, we align the measurement operators, noise levels, output resolution, and metric implementation. Since official sampler configurations of listed baselines differ in score evaluations, inner optimization steps, solvers, and backbones, we report computational cost alongside restoration quality in Tab.~\ref{main_results}. Unless otherwise stated, our method uses a single default hyperparameter
configuration across datasets and degradation types, without per-image or
per-baseline tuning. % Additional experimental results are provided in Appendix~\ref{add_results}, including trajectory visualizations, box-inpainting results, robustness analyses under real-world blur and operator mismatch, Stable-Diffusion-based \(512\times512\) experiments, parameter sensitivity studies, boundary-condition analyses, and additional visualization.

\begin{table*}[!h]
\caption{\textbf{Quantitative results on four linear inverse problems.}
The best and second-best results are marked in $\boldsymbol{bold}$ and $\underline{underlined}$, respectively. Runtime (s/image), peak memory (GB), and diffusion NFE are reported in the rightmost columns.} % We follow the evaluation protocol of \cite{daps} and compare against representative training-free posterior samplers and diffusion-prior solvers. 
\centering
\small
\setlength{\tabcolsep}{2.5pt}
\renewcommand{\arraystretch}{1.15}
\resizebox{\textwidth}{!}{%
\begin{tabular}
{|l|l|rrr|rrr|rrr|rrr|rrr|}
\toprule
\multirow{2}{*}{Dataset} & \multirow{2}{*}{Method}
& \multicolumn{3}{c|}{$4\times$ Super-resolution}
& \multicolumn{3}{c|}{Motion Deblurring}
& \multicolumn{3}{c|}{Gaussian Deblurring}
& \multicolumn{3}{c|}{Inpainting (random)}
& \multicolumn{3}{c|}{Efficiency}
\\
\cmidrule(lr){3-5}
\cmidrule(lr){6-8}
\cmidrule(lr){9-11}
\cmidrule(lr){12-14}
\cmidrule(lr){15-17}
& & PSNR$\uparrow$ & SSIM$\uparrow$ & LPIPS$\downarrow$
  & PSNR$\uparrow$ & SSIM$\uparrow$ & LPIPS$\downarrow$
  & PSNR$\uparrow$ & SSIM$\uparrow$ & LPIPS$\downarrow$
  & PSNR$\uparrow$ & SSIM$\uparrow$ & LPIPS$\downarrow$
  & Time & Mem. & NFE \\
\midrule

\multirow{8}{*}{FFHQ}
& DAPS
& \underline{29.35} & 0.782 & 0.193
& 29.66 & 0.847 & \underline{0.157}
& \underline{29.19} & \underline{0.817} & \textbf{0.165}
& \underline{30.72} & 0.800 & 0.159
& 80 & \underline{2.28} & 1000 \\

& DCDP
& 28.66 & 0.807 & \underline{0.178}
& 25.08 & 0.512 & 0.364
& 27.50 & 0.699 & 0.304
& 30.69 & \underline{0.842} & 0.142
& \textbf{2} & 4.93 & 181 \\

& DDNM
& 28.03 & 0.795 & 0.197
& 31.12 & 0.871 & 0.152
& 28.20 & 0.804 & 0.216
& 29.91 & 0.817 & \underline{0.121}
& \underline{5} & 4.70 & \underline{100} \\

& DDRM
& 26.58 & 0.782 & 0.282
& \underline{31.54} & \underline{0.853} & 0.173
& 24.93 & 0.732 & 0.239
& 28.46 & 0.822 & 0.183
& \underline{5} & 4.77 & \textbf{20} \\

& DPS
& 25.86 & 0.753 & 0.269
& 24.52 & 0.801 & 0.246
& 25.87 & 0.764 & 0.219
& 25.46 & 0.823 & 0.203
& 65 & 9.44 & 1000 \\

& FGPS
& - & - & -
& 21.70 & 0.574 & 0.294
& 22.01 & 0.601 & 0.267
& - & - & -
& 110 & 4.65 & 1000 \\

& FPS
& 28.42 & \underline{0.813} & 0.204
& 27.39 & 0.826 & 0.227
& 26.54 & 0.773 & 0.253
& 28.21 & 0.823 & 0.261
& 100 & 21.56 & 1000 \\

& ReSample
& 23.29 & 0.594 & 0.392
& 27.41 & 0.823 & 0.198
& 26.39 & 0.714 & 0.255
& 29.61 & 0.746 & 0.140
& 780 & 6.77 & 500 \\

\cmidrule(lr){2-17}
& Ours
& \textbf{31.88} & \textbf{0.890} & \textbf{0.090}
& \textbf{36.69} & \textbf{0.940} & \textbf{0.054}
& \textbf{29.99} & \textbf{0.821} & \underline{0.175}
& \textbf{32.86} & \textbf{0.909} & \textbf{0.066}
& 130 & \textbf{2.23} & 1000 \\

\midrule

\multirow{8}{*}{ImageNet}
& DAPS
& \underline{25.47} & 0.637 & \underline{0.294}
& 27.86 & 0.766 & 0.196
& 25.89 & 0.658 & 0.268
& \underline{28.44} & 0.775 & 0.135
& 160 & \underline{4.68} & 1000 \\

& DCDP
& 25.17 & 0.688 & 0.213
& - & - & -
& 26.08 & \textbf{0.727} & \textbf{0.195}
& 20.55 & \underline{0.864} & \underline{0.101}
& \textbf{5} & 12.55 & 181 \\

& DDNM
& 23.61 & 0.658 & 0.428
& 28.41 & 0.791 & 0.230
& \textbf{28.06} & 0.703 & 0.278
& 23.64 & 0.821 & 0.107
& \underline{15} & 6.37 & \underline{100} \\

& DDRM
& 24.13 & 0.633 & 0.304
& 28.48 & 0.784 & 0.182
& \underline{26.33} & \underline{0.713} & \underline{0.243}
& 26.34 & 0.764 & 0.181
& \textbf{5} & 6.44 & \textbf{20} \\

& DPS
& 21.13 & 0.489 & 0.361
& 18.96 & 0.629 & 0.423
& 20.31 & 0.598 & 0.397
& 23.52 & 0.745 & 0.297
& 200 & 9.44 & 1000 \\

& FGPS
& - & - & -
& 21.87 & 0.560 & 0.288
& 23.53 & 0.583 & 0.305
& - & - & -
& 375 & 11.08 & 1000 \\

& FPS
& 24.82 & \underline{0.703} & 0.313
& 24.52 & 0.647 & 0.326
& 23.91 & 0.601 & 0.387
& 24.52 & 0.701 & 0.316
& 150$^{*}$ & 57.66 & 1000 \\

& ReSample
& 22.61 & 0.576 & 0.370
& 26.94 & 0.738 & 0.227
& 25.97 & 0.703 & 0.254
& 27.50 & 0.756 & 0.143
& 780 & 6.77 & 500 \\

\cmidrule(lr){2-17}
& Ours
& \textbf{27.67} & \textbf{0.765} & \textbf{0.170}
& \textbf{34.70} & \textbf{0.935} & \underline{0.155}
& 26.02 & 0.672 & 0.283
& \textbf{28.60} & \textbf{0.887} & \textbf{0.042}
& 200 & \textbf{4.56} & 1000 \\

\bottomrule
\end{tabular}%
}
\begin{flushleft}
\footnotesize
$^{*}$Measured on an RTX Pro 6000 Blackwell due to A6000 memory limits, a GPU approximately \(2.5\times\) faster than the A6000.
\end{flushleft}
\vspace{-7mm}
\label{main_results}
\end{table*}

\subsection{Main Results}\label{sec_main}
We report quantitative results across four restoration tasks on both FFHQ and ImageNet in Tab.~\ref{main_results}. Overall, our method achieves comparable or superior performance across datasets and degradation types, with particularly large gains on \(4\times\) super-resolution, motion deblurring, and random inpainting. Although several baselines are competitive on ImageNet Gaussian deblurring, their advantages do not consistently transfer to other degradations, whereas our method maintains strong overall performance across diverse linear inverse problems. The ``-'' entries denote settings that are not natively supported by official implementations, and are omitted to avoid unreliable re-implementations. For the listed baselines, we adopt the default settings specified in their corresponding papers and open-source implementations, with a brief summary provided in Appendix~\ref{imple}.

Additional results are provided in Appendix~\ref{add_results}. Appendix~\ref{box} reports box-inpainting results, where our method remains competitive and achieves the best LPIPS on both FFHQ and ImageNet. Tab.~\ref{tab:same_budget_motion_deblur} reports a low-runtime setting, where a reduced-budget variant with \(N=9,T=90\) runs in approximately \(2\) seconds per image and outperforms the runtime-matched baseline across all metrics. Robustness analyses, Stable-Diffusion-based \(512\times512\) restoration experiments, parameter sensitivity studies, and additional qualitative comparisons are also included in Appendix~\ref{add_results}.

Complementing these quantitative results, we provide qualitative comparisons in Fig.~\ref{fig_main} using representative samples from both datasets. The FFHQ examples cover diverse facial attributes, expressions, and accessories, while the ImageNet examples include complex backgrounds, rich textures, and detailed object structures. Consistent with Tab.~\ref{main_results}, our method better preserves global
structure and object identity, restores sharper local textures, and suppresses
ringing artifacts and spurious high-frequency patterns under challenging
degradations.

To interpret the progressive frequency-exposure mechanism, we provide additional trajectory visualizations in Appendix~\ref{trajectory}, showing that the proposed continuation follows the intended coarse-to-fine behavior: early updates mainly stabilize low-frequency structure, while later updates progressively introduce higher-frequency details as the effective bandwidth expands. This supports the role of noise-conditioned frequency exposure in avoiding premature high-frequency commitment.

\begin{figure}[!h]
    \centering
    \vspace{-2mm}
    \includegraphics[width=0.65\linewidth]{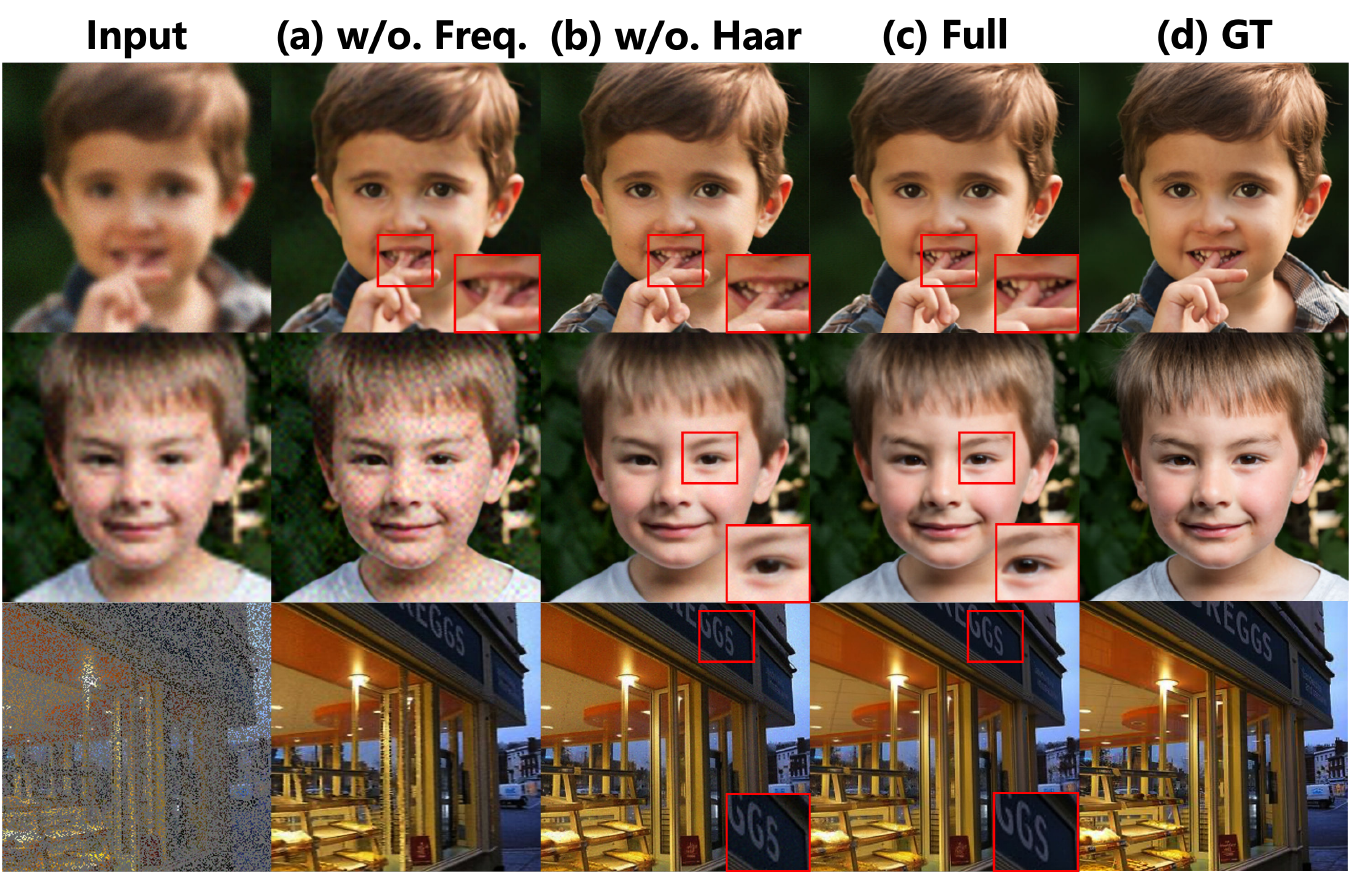}
    \caption{\textbf{Visual ablation of frequency continuation and Haar commitment.}
We visualize two ablation variants: column (a) replaces noise-conditioned frequency continuation with full-band guidance, and column (b) removes the Haar fusion step in Eq.~\ref{haar_fusion}. The first two rows are evaluated on FFHQ, and the third row is evaluated on ImageNet.}
    \label{fig_abl}
    \vspace{-4mm}
\end{figure}

\subsection{Ablation Studies}\label{sec:abl}
We ablate the proposed posterior sampling framework to isolate the contribution of each design choice and to verify the mechanistic claims in Sec.~\ref{log_density} and Sec.~\ref{sec_har}. Specifically, we examine frequency continuation, Haar commitment, and alternatives to the Haar-based fusion design. Parameter sensitivity is reported in Appendix~\ref{sensitivity}. % All ablation variants are evaluated using the same settings as the full method.

\textbf{Frequency continuation and Haar commitment.} We first ablate the two core components, as visualized in Fig.~\ref{fig_abl}. We remove frequency continuation by replacing the band-limited likelihood with full-band measurement enforcement, where \(\omega(\sigma_i)\) is set to a large fixed value and \(\lambda_i=0\) in Eq.~\eqref{eq:fg_loss}. We remove Haar commitment by directly outputting \(\boldsymbol{x}_0^{(T)}\) from Eq.~\ref{ula_update}, without the coarse-detail fusion in Eq.~\ref{haar_fusion}. Columns (a) and (b) in Fig.~\ref{fig_abl} visualize these two removals, respectively. Full-band guidance tends to introduce drift and operator-sensitive artifacts at high noise, consistent with Theorem~\ref{thm:unstable} and Corollary~\ref{sensitive}, while removing Haar fusion degrades detail quality, matching Lemma~\ref{lp_insensitive_detail} and Corollary~\ref{detail_prior_dominated}. These variants separate the two roles: frequency continuation controls trusted measurement directions, whereas Haar commitment controls reliable detail adoption. These observations are further supported by the quantitative results in Tab.~\ref{tab:abl}.

\begin{table*}[!h]
\caption{\textbf{Quantitative ablation of frequency continuation and Haar commitment.}
We compare the full model with variants that remove frequency continuation or Haar commitment on FFHQ.}
    \centering
    
    \resizebox{\textwidth}{!}{%
    \begin{tabular}{|l|ccc|ccc|ccc|ccc|}
    \toprule
    \multirow{3}{*}{\centering {Method}}
    & \multicolumn{3}{c|}{{Super-resolution}} 
    & \multicolumn{3}{c|}{{Motion Deblurring}} 
    & \multicolumn{3}{c|}{{Gaussian Deblurring}} 
    & \multicolumn{3}{c|}{{Inpainting}} \\ 
    \cmidrule(lr){2-4} \cmidrule(lr){5-7} \cmidrule(lr){8-10} \cmidrule(lr){11-13}
    & {PSNR}$\uparrow$  
    & {SSIM}$\uparrow$ & {LPIPS}$\downarrow$ 
    & {PSNR}$\uparrow$  
    & {SSIM}$\uparrow$ & {LPIPS}$\downarrow$
    & {PSNR}$\uparrow$  
    & {SSIM}$\uparrow$ & {LPIPS}$\downarrow$
    & {PSNR}$\uparrow$  
    & {SSIM}$\uparrow$ & {LPIPS}$\downarrow$ \\
    \midrule 
    w/o. Freq
    & 27.17 & 0.655 & 0.342 
    & 29.46 & 0.723 & 0.204
    & 24.15 & \underline{0.774} & \textbf{0.160}
    & 24.67 & 0.784 & 0.163  \\
    
    w/o. Haar 
    & \underline{29.60} & \underline{0.804} & \underline{0.186} 
    & \underline{32.78} & \underline{0.879} & \underline{0.132}
    & \underline{24.54} & 0.745 & 0.177
    & \underline{29.58} & \underline{0.841} & \underline{0.140} \\
    \midrule
    
    Full
    & \textbf{31.88} & \textbf{0.890} & \textbf{0.090} 
    & \textbf{36.69} & \textbf{0.940} & \textbf{0.054}
    & \textbf{29.99} & \textbf{0.821} & \underline{0.175}
    & \textbf{32.86} & \textbf{0.909} & \textbf{0.066} \\
    \bottomrule
    \end{tabular}
    }
    \label{tab:abl}
    \vspace{-2mm}
\end{table*}

\textbf{Haar basis and frequency-reweighting alternatives.} We further examine the design choice of Haar fusion and whether it can be replaced by simpler frequency-domain reweighting. We use motion deblurring as a representative ill-conditioned setting because it is particularly sensitive to high-frequency artifacts and operator-induced drift. Tab.~\ref{tab:wavelet_freq_ablation} shows that Haar achieves the best PSNR and SSIM on both datasets and the best overall trade-off among the tested Daubechies-2 and Biorthogonal 4.4 bases in this setting, with qualitative comparisons provided in Appendix~\ref{add_qual} and Fig.~\ref{fig_wavelet_basis}. The table also includes a frequency-dependent measurement-reweighting alternative without Haar fusion, which performs substantially worse than the full model. This suggests that pre-update spectral weighting alone is insufficient to replace post-refinement detail commitment.

\begin{table*}[!h]
\caption{\textbf{Ablation studies of wavelet bases and frequency-reweighted updates} on motion deblurring. The first block compares different wavelet bases for the commitment module, while the second block tests replacing Haar commitment with frequency-dependent measurement-consistency updates.}
\centering
\small
\setlength{\tabcolsep}{3.5pt}
\renewcommand{\arraystretch}{1.15}
\resizebox{0.8\textwidth}{!}{%
\begin{tabular}{|l|l|rrr|rrr|}
\toprule
\multirow{2}{*}{Ablation} & \multirow{2}{*}{Method}
& \multicolumn{3}{c|}{FFHQ}
& \multicolumn{3}{c|}{ImageNet}
\\
\cmidrule(lr){3-5}
\cmidrule(lr){6-8}
& & PSNR$\uparrow$ & SSIM$\uparrow$ & LPIPS$\downarrow$
  & PSNR$\uparrow$ & SSIM$\uparrow$ & LPIPS$\downarrow$ \\
\midrule

\multirow{3}{*}{Wavelet basis}
& Daubechies-2
& 26.82 & \underline{0.890} & \underline{0.119}
& 24.26 & \underline{0.856} & \textbf{0.142} \\

& Biorthogonal 4.4
& \underline{30.59} & 0.876 & 0.151
& \underline{27.64} & 0.811 & 0.186 \\

& Ours (Haar)
& \textbf{36.69} & \textbf{0.940} & \textbf{0.054}
& \textbf{34.70} & \textbf{0.935} & \underline{0.155} \\

\midrule

\multirow{2}{*}{Freq. reweighting}
& Freq.-weighted update w/o Haar
& 14.88 & 0.179 & 0.583
& 15.02 & 0.244 & 0.482 \\

& Ours (Haar + Freq.)
& \textbf{36.69} & \textbf{0.940} & \textbf{0.054}
& \textbf{34.70} & \textbf{0.935} & \textbf{0.155} \\

\bottomrule
\end{tabular}%
}
\label{tab:wavelet_freq_ablation}
\vspace{-3mm}
\end{table*}

\section{Conclusion}

We presented a posterior continuation framework for diffusion inverse problems
based on noise-conditioned frequency exposure. The central principle is to expose
measurement information according to the instantaneous diffusion noise level,
rather than enforcing full-band consistency uniformly throughout sampling. This
defines a sequence of intermediate posteriors that emphasize reliable frequency
bands at high noise and gradually return to full-band measurement consistency as
the noise decreases. To realize this principle, we introduced a training-free
sampler that combines diffusion prediction, frequency-limited likelihood
refinement, and a post-refinement Haar-domain commitment rule. The Haar
commitment synchronizes the adoption of refined details with the exposure
schedule, committing reliable coarse corrections early while deferring weakly
identifiable details. Experiments on super-resolution, inpainting, and Gaussian
and motion deblurring demonstrate strong restoration performance across datasets
and operators, with particularly large gains on ill-conditioned motion blur.
Ablations further validate the coupled design of frequency exposure and Haar
commitment, showing that stable posterior sampling benefits from both
noise-conditioned measurement guidance and schedule-synchronized detail
commitment.

\bibliographystyle{plainnat}
\bibliography{example_paper}

%%%%%%%%%%%%%%%%%%%%%%%%%%%%%%%%%%%%%%%%%%%%%%%%%%%%%%%%%%%%

\appendix
\newpage
\appendix
\onecolumn
\section{Appendix}
\subsection{Related Work}\label{related}
Recent posterior sampling methods for inverse problems increasingly adopt diffusion and score based generative models as expressive implicit priors. Foundational results connect reverse time diffusion to score estimation and denoising, including Reverse time diffusion equation models~\cite{rela_0} and A connection between score matching and denoising autoencoders~\cite{rela_1}. Modern diffusion backbones include Denoising diffusion probabilistic models~\cite{ddpm} and implicit sampling via Denoising diffusion implicit models~\cite{ddim}, with a continuous time formulation in Score based generative modeling through stochastic differential equations~\cite{scoresde}. Sampling quality and efficiency further benefit from EDM~\cite{score3}, fast ODE solvers such as DPM-Solver~\cite{dpmsolver}, and guidance mechanisms such as Classifier-free diffusion guidance~\cite{cfg}. These advances make diffusion priors practical plug in components for Bayesian posterior inference under ill posed degradations.

Prior to diffusion priors, inverse problems were largely addressed by explicit regularization and convex recovery, exemplified by Stable signal recovery from incomplete and inaccurate measurements~\cite{rela_2} and Exact matrix completion via convex optimization~\cite{rela_3}. Learned generative priors then enabled model based recovery with neural generators, including Compressed sensing using generative models~\cite{rela_4} and high fidelity GAN backbones such as A style based generator architecture for generative adversarial networks~\cite{ffhq}. Plug and play optimization further linked model based solvers with learned denoisers, with convergence oriented instantiations such as Plug and play ADMM for image restoration Fixed point convergence and applications~\cite{rela_6} grounded in ADMM~\cite{rela_7}. Related directions include algorithm unrolling~\cite{rela_8} and posterior sampling toolkits beyond MAP, such as SGLD~\cite{lgvd}, HMC~\cite{rela_10}, and Randomize then optimize~\cite{rela_11}. This trajectory motivates diffusion priors as a stronger learned prior that still interfaces naturally with likelihood constrained inference.

Building on these foundations, diffusion based posterior sampling directly targets Bayesian inference with diffusion priors, including DPS~\cite{dps}. Sequential viewpoints formalize diffusion inverse sampling through filtering style inference~\cite{fps}, while other efforts improve likelihood incorporation and posterior score surrogates, including Monte Carlo guided variants~\cite{rela_12} and Tweedie moment--based projected dynamics~\cite{rela_13}. To address local correction induced by small reverse time steps, decoupling strategies such as DAPS~\cite{daps} enable larger and more non local posterior moves. In parallel, projection and data consistency based diffusion restorers remain influential, including DDRM~\cite{ddrm} and operator aware null space formulations such as DDNM~\cite{ddnm}. For conditional editing and inpainting, representative mechanisms include RePaint~\cite{rela_14} and ILVR~\cite{ilvr}.

Beyond standard degradations, robustness and distribution shift are increasingly emphasized. Steerable conditional diffusion explicitly targets OOD inverse problems~\cite{rela_15}, aligning with broader robustness benchmarks~\cite{rela_16}. Equilibrium formulations provide an alternative lens via deep equilibrium architectures~\cite{rela_17}. Recent surveys consolidate diffusion inverse methods and taxonomies~\cite{rela_18}.

\subsection{Pseudo-code of The Algorithm}\label{add_results}
We provide the corresponding Pseudo-code for our method in Alg.~\ref{alg:main}.
\begin{algorithm}[t]
\caption{Noise-Conditioned Posterior Continuation}
\label{alg:main}
\begin{algorithmic}[1]
\STATE \textbf{Require:} Diffusion model $\varepsilon_\theta$, clean-estimate solver $f$, measurement $(\boldsymbol{y},\mathcal{A})$, 
number of outer sampling steps $N$, sampling times $\{\tau_i\}_{i=0}^{N}$, noise levels $\{\sigma_i\}_{i=0}^{N}$ with $\sigma_i=\sigma(\tau_i)$ and $\sigma_N>\cdots>\sigma_0=0$, 
exposure schedules $\omega(\sigma)$ and $\lambda(\sigma)$, temperature schedule $\{\beta_i\}_{i=1}^{N}$, 
Langevin step sizes $\{\eta_{i,\ell}\}$, Langevin steps $T$, Haar transform $W$, Haar gate parameters $d_s,d_e,\gamma$.

\STATE Initialize $\boldsymbol{x}_{\tau_N}\sim\mathcal{N}(\boldsymbol{0},\sigma_N^2\boldsymbol{I})$.

\FOR{$i=N$ to $1$}
    \STATE $\boldsymbol{\hat{x}}_{0,i} \leftarrow f(\boldsymbol{x}_{\tau_i},\varepsilon_\theta,\sigma_i)$
    \STATE \emph{// Diffusion predictor. Estimate the clean image at the current noise level.}

    \STATE $\omega_i \leftarrow \omega(\sigma_i)$, $\lambda_i \leftarrow \lambda(\sigma_i)$
    \STATE $\mathcal{L}_i(\boldsymbol{x};\boldsymbol{y}) \leftarrow 
    (1-\lambda_i)\|\mathcal{A}(\boldsymbol{x})-\boldsymbol{y}\|_2^2
    +\lambda_i\|P^F_{\omega_i}(\mathcal{A}(\boldsymbol{x})-\boldsymbol{y})\|_2^2$

    \STATE $\pi_i(\boldsymbol{x}_0) \propto 
    \exp\left(-\frac{1}{2\beta_i^2}\mathcal{L}_i(\boldsymbol{x}_0;\boldsymbol{y})\right)
    \mathcal{N}(\boldsymbol{x}_0;\boldsymbol{\hat{x}}_{0,i},\sigma_i^2\boldsymbol{I})$
    \STATE \emph{// Frequency exposure. Unreliable high-frequency residuals are attenuated.}

    \STATE $\boldsymbol{x}_{0,i}^{(0)} \leftarrow \boldsymbol{\hat{x}}_{0,i}$
    \FOR{$\ell=0$ to $T-1$}
        \STATE Sample $\boldsymbol{\xi}_{i,\ell}\sim\mathcal{N}(\boldsymbol{0},\boldsymbol{I})$
        \STATE $\boldsymbol{x}_{0,i}^{(\ell+1)} \leftarrow 
        \boldsymbol{x}_{0,i}^{(\ell)}
        +\eta_{i,\ell}\nabla_{\boldsymbol{x}_0}\log\pi_i(\boldsymbol{x}_{0,i}^{(\ell)})
        +\sqrt{2\eta_{i,\ell}}\boldsymbol{\xi}_{i,\ell}$
    \ENDFOR
    \STATE $\boldsymbol{x}_{0,i}^{\mathrm{ref}}\leftarrow \boldsymbol{x}_{0,i}^{(T)}$
    \STATE \emph{// Langevin refinement under the exposed likelihood.}

    \STATE $\boldsymbol{\hat{z}}_i \leftarrow W\boldsymbol{\hat{x}}_{0,i}$, 
    $\boldsymbol{z}_i^{\mathrm{ref}}\leftarrow W\boldsymbol{x}_{0,i}^{\mathrm{ref}}$
    \STATE $\rho_i\leftarrow (N-i)/(N-1)$
    \STATE $w_{i,\mathrm{c}}\leftarrow 1$, 
    $w_{i,\mathrm{d}}\leftarrow \left[d_s+(d_e-d_s)\rho_i^\gamma\right](1-\lambda_i)$
    \STATE $\boldsymbol{z}_{i,b}^{\mathrm{fuse}}\leftarrow 
    (1-w_{i,b})\boldsymbol{\hat{z}}_{i,b}+w_{i,b}\boldsymbol{z}_{i,b}^{\mathrm{ref}},
    \quad b\in\{\mathrm{c},\mathrm{d}\}$
    \STATE $\boldsymbol{x}_{0,i}^{\mathrm{fuse}}\leftarrow W^\top\boldsymbol{z}_i^{\mathrm{fuse}}$
    \STATE \emph{// Post-refinement Haar commitment. Coarse corrections are committed, details are gated.}

    \STATE $\boldsymbol{x}_{\tau_{i-1}}\sim
    \mathcal{N}(\boldsymbol{x}_{0,i}^{\mathrm{fuse}},\sigma_{i-1}^2\boldsymbol{I})$
    \STATE \emph{// Re-noise the committed estimate for the next reverse step.}
\ENDFOR

\STATE \textbf{return} $\boldsymbol{x}_{0,1}^{\mathrm{fuse}}$
\end{algorithmic}
\end{algorithm}

% Experiments that report how quantitative metrics vary with the sampling steps for Gaussian deblurring are shown in Fig.~\ref{fig_app_1}.

% More generated results on FFHQ $256\times 256$ and ImageNet $256\times 256$ are shown in Fig.~\ref{fig_app_2} and Fig.~\ref{fig_app_3} repectively.

\subsection{Additional Experimental Results}\label{add_results}

\subsubsection{Trajectory Visualization of Noise-Conditioned Frequency Exposure}\label{trajectory}
To better understand the sampling dynamics, we visualize the intermediate trajectory induced by the noise-conditioned frequency exposure schedule. Fig.~\ref{fig_exp_1}(a) shows the masked Fourier magnitude maps at different stages, where the exposed bandwidth gradually expands as the diffusion noise decreases. Fig.~\ref{fig_exp_1}(b) presents the corresponding spatial-domain reconstructions. At early stages, only low-frequency measurement directions are exposed, so the reconstruction mainly captures coarse layout, silhouette, and global illumination. As the cutoff frequency increases, higher-frequency components are progressively introduced, leading to sharper edges, textures, and fine-scale details. Fig.~\ref{fig_exp_1}(c) further visualizes the sampling trajectory and shows that the reconstruction evolves consistently with the frequency-exposure schedule.

\begin{figure}[!h]
    \centering
    \includegraphics[width=0.75\linewidth]{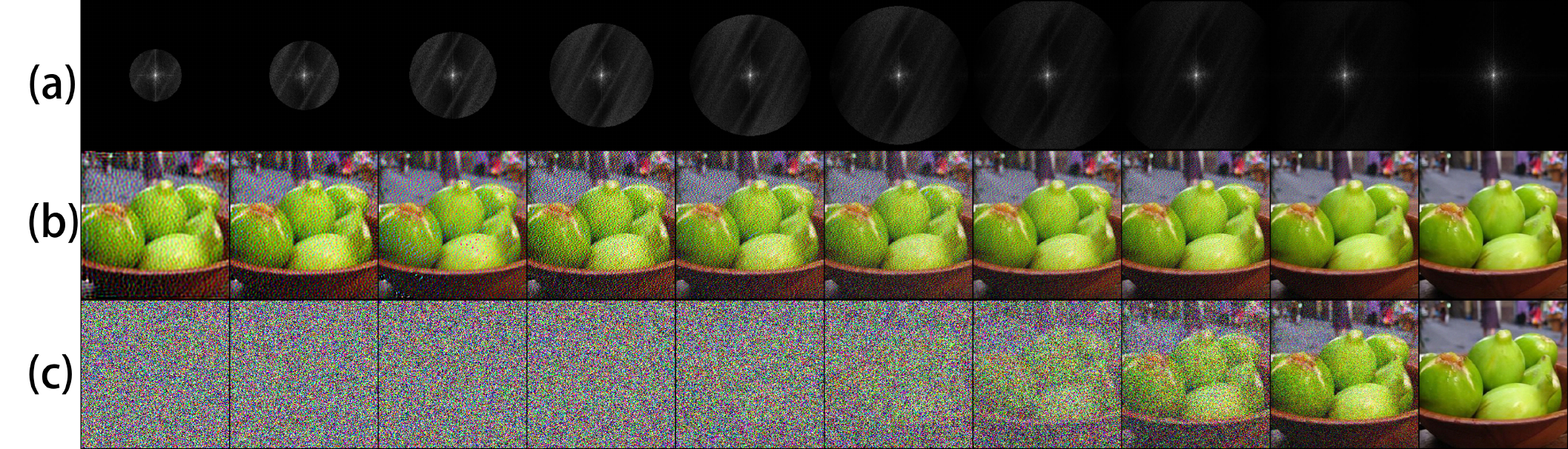}
    \caption{\textbf{Example visualization of noise-conditioned frequency exposure on ImageNet.}
    (a) Masked Fourier magnitude maps induced by the cutoff-frequency schedule.
    (b) Spatial-domain reconstructions along the progressive frequency-exposure trajectory.
    (c) Corresponding sampling trajectory. }
    \label{fig_exp_1}
\end{figure}

We also track quantitative metrics along the sampling process in Fig.~\ref{fig_exp_2} and Fig.~\ref{fig_app_1}. Fig.~\ref{fig_exp_2} reports the metric evolution for FFHQ motion deblurring, while Fig.~\ref{fig_app_1} shows the corresponding results for FFHQ Gaussian deblurring. In both figures, the orange curves report metrics computed from the intermediate estimates produced by the refinement update in Eq.~\ref{ula_update}, while the red/green curves report metrics computed from the denoised estimates after each solver step. Overall, PSNR and SSIM improve while LPIPS decreases along the trajectory, indicating that the proposed sampler progressively enhances both measurement consistency and perceptual quality. These results further support the intended coarse-to-fine behavior of our continuation process: early updates stabilize reliable low-frequency structure, while later updates commit increasingly fine details as they become more identifiable.

\begin{figure*}[!h]
    \centering
    \includegraphics[width=\linewidth]{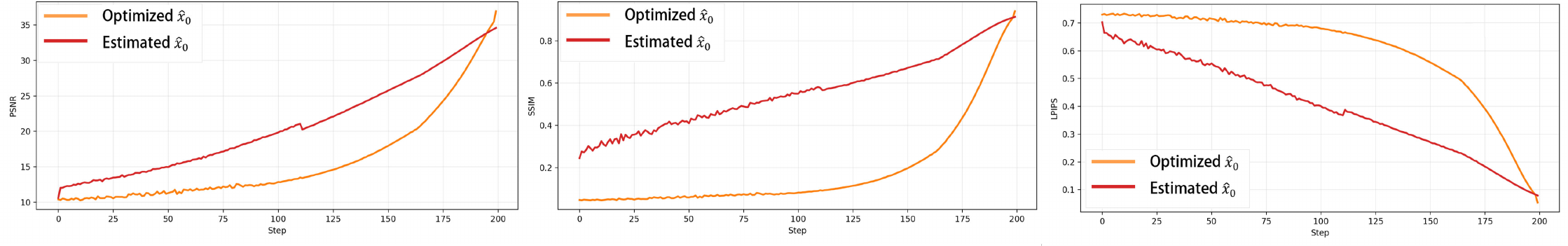}
    \caption{\textbf{Metric evolution during sampling for FFHQ motion deblurring.}
    The orange curves report metrics computed from the intermediate estimates produced by the refinement update in Eq.~\ref{ula_update}, while the red curves report metrics computed from the denoised estimates after each solver step.}
    \label{fig_exp_2}
\end{figure*}

\begin{figure*}[!h]
    \centering
    \includegraphics[width=\linewidth]{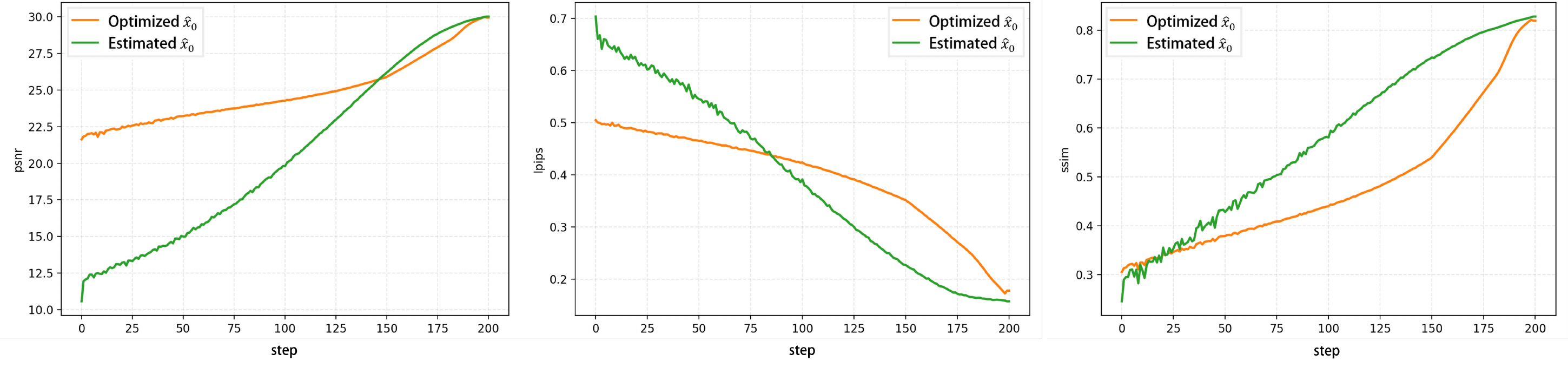}
    \vspace{-15pt}
    \caption{\textbf{Metric evolution during sampling for FFHQ Gaussian deblurring.}
    The orange curves report metrics computed from the intermediate estimates produced by the refinement update in Eq.~\ref{ula_update}, while the green curves report metrics computed from the denoised estimates after each solver step.}
    \label{fig_app_1}
\end{figure*}

\subsubsection{Additional Results on Box Inpainting}\label{box}

We further evaluate our method on box inpainting, where the missing region is spatially contiguous. As shown in Tab.~\ref{tab:box_inpainting}, our method remains competitive and achieves the best LPIPS on both FFHQ and ImageNet. It also obtains the best SSIM on FFHQ and stays close to the strongest baselines in PSNR and SSIM on ImageNet. The gains are smaller than in deblurring or super-resolution because the box mask only constrains visible context, leaving the missing region unobserved. Since masking mixes frequencies in the Fourier domain, the low-frequency residual mainly enforces visible-region consistency and boundary compatibility rather than revealing the true masked structure. Thus, Haar fusion cannot create missing semantics, and it primarily prevents premature commitment to unreliable details, as qualitatively illustrated in Fig.~\ref{fig:box_inpainting_vis}.

\begin{table*}[!h]
\caption{\textbf{Quantitative results on box inpainting.}
We report results on FFHQ and ImageNet at \(256\times256\). The best and second-best results are marked in \(\boldsymbol{bold}\) and \(\underline{underlined}\), respectively.}
\centering
\small
\setlength{\tabcolsep}{4.5pt}
\renewcommand{\arraystretch}{1.15}
\resizebox{0.7\textwidth}{!}{%
\begin{tabular}{|l|rrr|rrr|}
\toprule
\multirow{2}{*}{Method}
& \multicolumn{3}{c|}{FFHQ}
& \multicolumn{3}{c|}{ImageNet} \\
\cmidrule(lr){2-4}
\cmidrule(lr){5-7}
& PSNR$\uparrow$ & SSIM$\uparrow$ & LPIPS$\downarrow$
& PSNR$\uparrow$ & SSIM$\uparrow$ & LPIPS$\downarrow$ \\
\midrule
DAPS
& 24.07 & 0.814 & 0.133
& 21.43 & 0.725 & 0.214 \\

DCDP
& 23.89 & 0.760 & 0.163
& - & - & - \\

DDNM
& 24.47 & \underline{0.837} & 0.235
& 21.64 & \textbf{0.748} & 0.319 \\

DDRM
& 22.26 & 0.801 & 0.207
& 18.63 & 0.733 & 0.254 \\

DPS
& 22.51 & 0.792 & 0.209
& 18.94 & 0.722 & 0.257 \\

FPS
& \textbf{24.86} & 0.823 & 0.146
& \textbf{22.16} & 0.726 & \underline{0.208} \\

ReSample
& 20.06 & 0.749 & 0.184
& 18.29 & 0.631 & 0.262 \\

% LatentDAPS
% & 23.99 & 0.802 & 0.194
% & 17.19 & 0.624 & 0.340 \\

% PSLD
% & 24.22 & 0.813 & 0.158
% & 20.10 & 0.694 & 0.465 \\

Ours
& \underline{24.57} & \textbf{0.838} & \textbf{0.126}
& \underline{21.58} & \underline{0.739} & \textbf{0.207} \\
\bottomrule
\end{tabular}%
}
\label{tab:box_inpainting}
\end{table*}

\begin{figure}[t]
    \centering
    \begin{minipage}{0.49\linewidth}
        \centering
        \includegraphics[width=\linewidth]{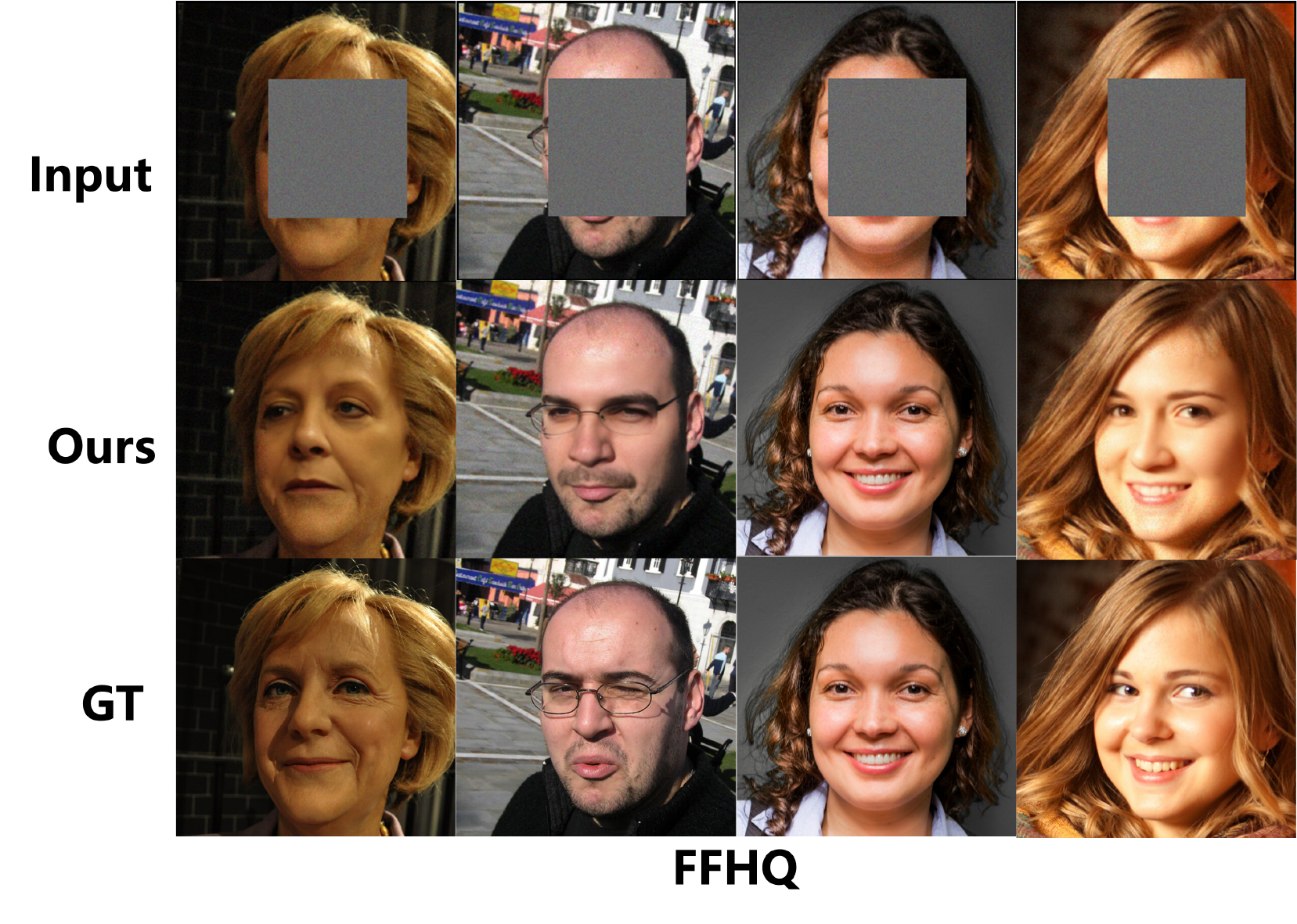}
    \end{minipage}
    \hfill
    \begin{minipage}{0.49\linewidth}
        \centering
        \includegraphics[width=\linewidth]{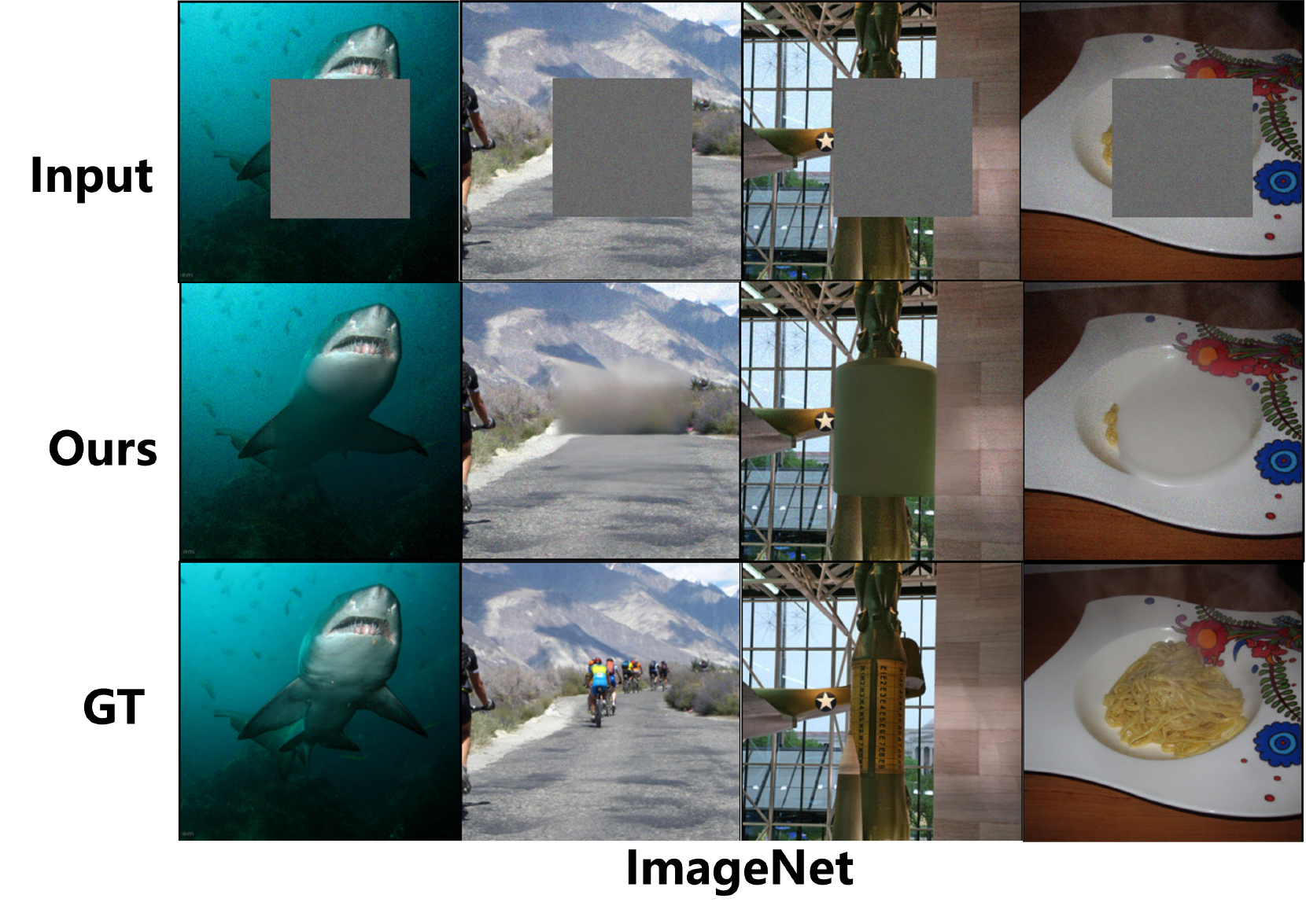}
    \end{minipage}
    \caption{\textbf{Qualitative examples of box inpainting.} Box masks leave the missing region unobserved, making semantic recovery more challenging than deblurring or super-resolution.}
    \label{fig:box_inpainting_vis}
\end{figure}

\subsubsection{Robustness, Scalability, and Boundary Conditions}

We further evaluate the robustness, scalability, and boundary conditions of our design under real-world blur, \(512\times512\) latent-space restoration, and operator mismatch. These experiments clarify when the proposed continuation remains effective and where its assumptions become limiting. Overall, our method is most effective when the assumed likelihood is approximately correct and reliable measurements follow a coarse-to-fine structure.

\paragraph{Real-world blur.} For real-world blur, we follow BlindDPS \cite{blinddps}. We provide qualitative examples under real-world blur in Fig.~\ref{fig:real_blur_vis}. The results show that our method remains robust beyond matched synthetic degradations and still outperforms the compared baseline.

\begin{figure}[!h]
    \centering
    \begin{minipage}{0.49\linewidth}
        \centering
        \includegraphics[width=\linewidth]{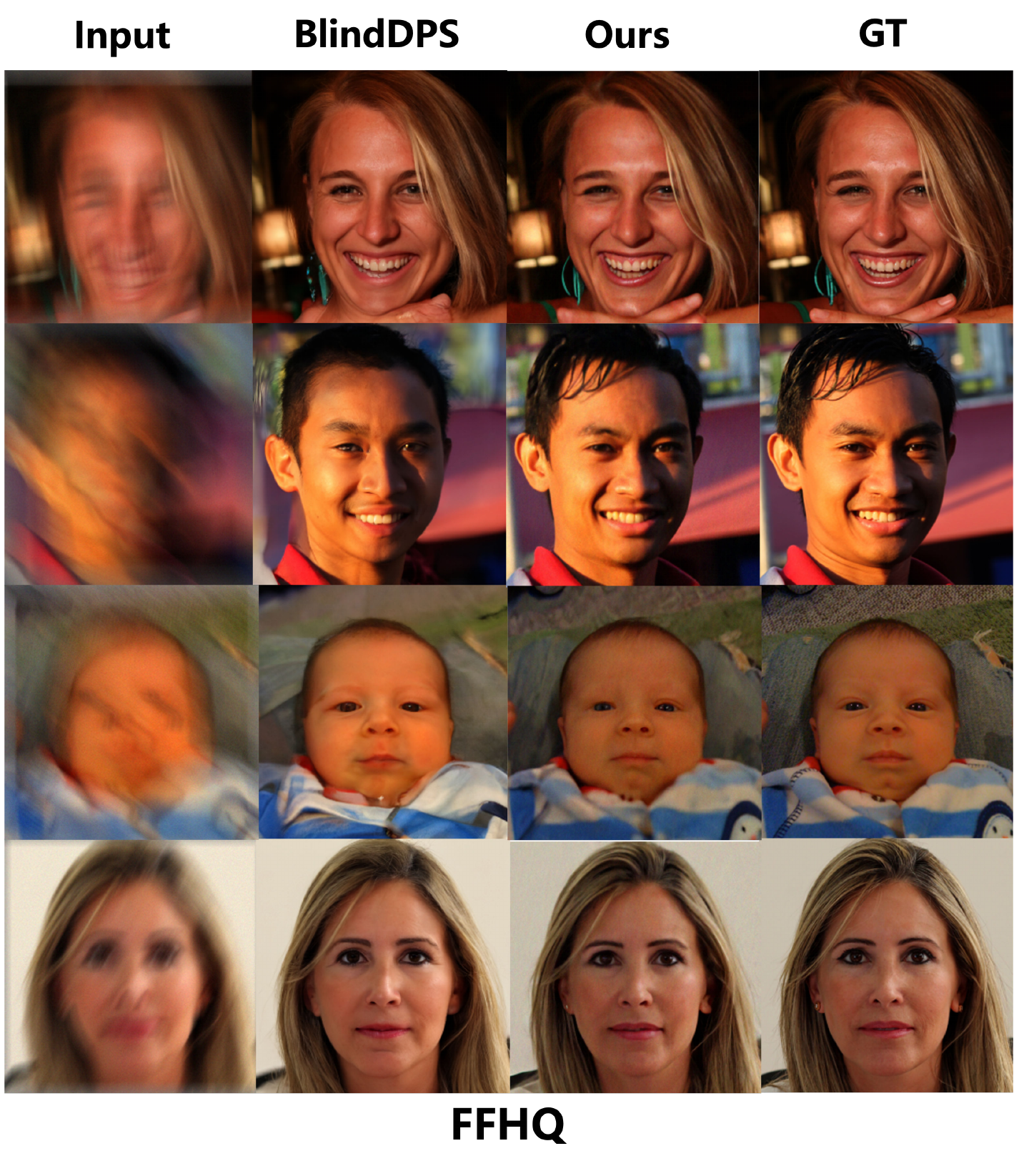}
    \end{minipage}
    \hfill
    \begin{minipage}{0.50\linewidth}
        \centering
        \includegraphics[width=\linewidth]{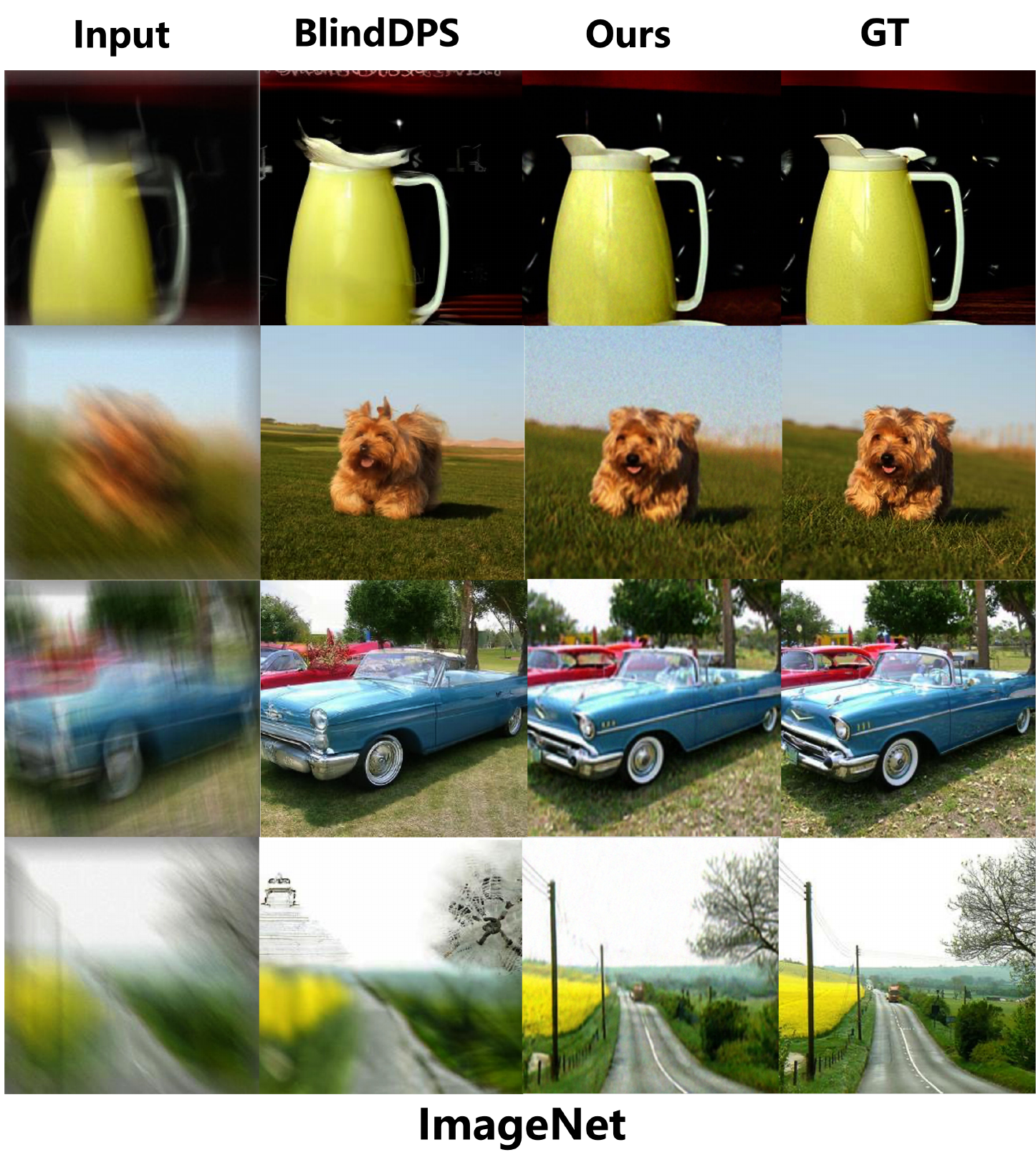}
    \end{minipage}
    \caption{\textbf{Qualitative examples of real-world blur.}
    Our method remains robust under real-world blur and produces clearer restorations than the compared baseline.}
    \label{fig:real_blur_vis}
\end{figure}

\paragraph{Stable-Diffusion-based restoration at \(512\times512\).}
Since different pretrained diffusion models have different native resolution preferences, we select Stable Diffusion v1.5 as a representative latent diffusion backbone for \(512\times512\) restoration. To verify the effectiveness of our method for higher-resolution restoration in latent space, we compare it with the recent SILO method~\cite{silo}, which also adopts Stable Diffusion v1.5 and operates at \(512\times512\) resolution. Tab.~\ref{tab:silo_ours_ffhq_512} and Fig.~\ref{silo} report the quantitative and qualitative comparison between SILO and our method. Compared with pixel-space restoration, latent-space posterior sampling is more sensitive to posterior geometry, decoder distortion, representation bottlenecks, and text-conditioning effects.

For box inpainting, the improvement is smaller because a contiguous mask only constrains the visible context and does not directly reveal the missing-region semantics. Nevertheless, our method remains competitive, achieving better SSIM and LPIPS than SILO under the aligned \(512\times512\) setting.

\begin{figure}[!h]
    \centering
    \includegraphics[width=0.6\linewidth]{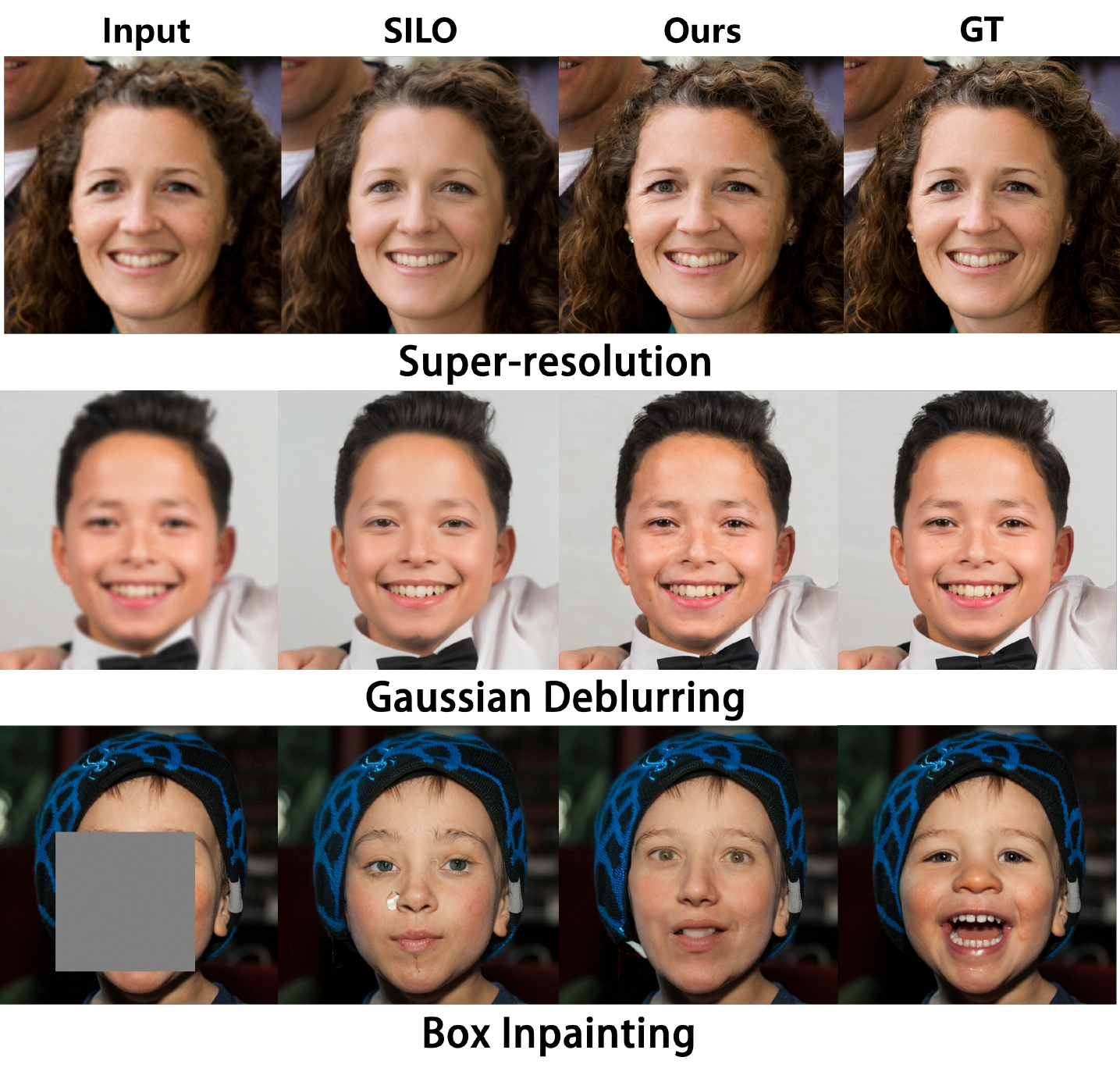}
    \caption{\textbf{Qualitative comparison with SILO on FFHQ \(512\times512\).}
Both methods use Stable Diffusion v1.5 as the backbone. Our method produces competitive restorations with sharper visible details in these examples.}
    \label{silo}
\end{figure}

\begin{table*}[!h]
\caption{\textbf{Quantitative comparison with SILO on FFHQ (\(512\times512\)).}
We report results on three restoration tasks.}
\centering
\small
\setlength{\tabcolsep}{4.5pt}
\renewcommand{\arraystretch}{1.15}
\resizebox{0.95\textwidth}{!}{%
\begin{tabular}{|l|rrr|rrr|rrr|}
\toprule
\multirow{2}{*}{Methods}
& \multicolumn{3}{c|}{Gaussian Deblurring}
& \multicolumn{3}{c|}{Super-resolution}
& \multicolumn{3}{c|}{Box Inpainting} \\
\cmidrule(lr){2-4}
\cmidrule(lr){5-7}
\cmidrule(lr){8-10}
& PSNR$\uparrow$ & SSIM$\uparrow$ & LPIPS$\downarrow$
& PSNR$\uparrow$ & SSIM$\uparrow$ & LPIPS$\downarrow$
& PSNR$\uparrow$ & SSIM$\uparrow$ & LPIPS$\downarrow$ \\
\midrule
SILO \((512 \times 512)\)
& \underline{26.53} & \underline{0.745} & \underline{0.317}
& \underline{26.79} & \underline{0.761} & \underline{0.297}
& \textbf{22.73} & \underline{0.782} & \underline{0.237} \\

Ours \((512 \times 512)\)
& \textbf{27.29} & \textbf{0.762} & \textbf{0.232}
& \textbf{28.04} & \textbf{0.803} & \textbf{0.177}
& \underline{22.62} & \textbf{0.785} & \textbf{0.232} \\
\bottomrule
\end{tabular}%
}
\label{tab:silo_ours_ffhq_512}
\end{table*}

\paragraph{Operator mismatch.}
We next examine the boundary of our method under operator mismatch. The observation is generated by the true blur operator \(A_{\mathrm{true}}\), with kernel size \(61\) and intensity \(0.50\), while reconstruction uses a perturbed operator \(A_{\mathrm{mismatch}}\). As shown in Tab.~\ref{tab:operator_mismatch}, performance drops substantially when the assumed operator deviates from the true one. This is expected because the sampler targets a biased posterior induced by the misspecified likelihood rather than the posterior induced by \(A_{\mathrm{true}}\). The measurement-consistency term is effectively computed as
\[
\|A_{\mathrm{mismatch}}(\hat{\boldsymbol{x}}_0)-A_{\mathrm{true}}(\boldsymbol{x}_{\mathrm{gt}})\|_2^2,
\]
where \(\hat{\boldsymbol{x}}_0\) and \(\boldsymbol{x}_{\mathrm{gt}}\) denote the estimated and ground-truth images. Thus, our method can stabilize sampling under a given likelihood, but it cannot correct a systematically misspecified forward model. This represents a natural failure mode when the assumed degradation operator is substantially wrong.

\begin{table}[t]
\caption{\textbf{Operator mismatch on FFHQ motion deblurring.}
The observation is generated with the true operator of kernel size \(61\) and intensity \(0.50\), while reconstruction uses a perturbed operator.}
\vspace{3mm}
\centering
\small
\setlength{\tabcolsep}{4.5pt}
\renewcommand{\arraystretch}{1.15}
\resizebox{0.75\columnwidth}{!}{%
\begin{tabular}{|l|ccc|}
\toprule
Assumed operator
& PSNR$\uparrow$ & SSIM$\uparrow$ & LPIPS$\downarrow$ \\
\midrule
Kernel size 51, intensity 0.05
& 19.36 & 0.477 & 0.416 \\

Kernel size 71, intensity 0.05
& 15.20 & 0.286 & 0.526 \\

Kernel size 61, intensity 0.65
& 9.69 & 0.155 & 0.689 \\

Kernel size 61, intensity 0.08
& 10.03 & 0.150 & 0.672 \\

Kernel size 61, intensity 0.50 (Matched)
& \textbf{36.69} & \textbf{0.940} & \textbf{0.054} \\
\bottomrule
\end{tabular}%
}
\label{tab:operator_mismatch}
\end{table}

\subsubsection{Parameter Sensitivity}\label{sensitivity}

We analyze the sensitivity of key sampling and fusion parameters on FFHQ motion deblurring. We first study the interaction between the continuation steps \(N\) and inner refinement steps \(T\), which controls the quality--cost trade-off. We then examine the effects of the measurement noise level \(\sigma_y\) and representative continuation and fusion parameters.

\paragraph{Sampling budget \(N,T\).} Tab.~\ref{tab:nt_sensitivity} reports the joint sensitivity to the number of continuation steps \(N\) and inner refinement steps \(T\). Increasing \(N\) makes the continuation trajectory finer, so the transition from low-frequency guidance to full-band consistency becomes smoother. Increasing \(T\) improves local refinement at each fixed stage. The two factors are complementary: a large \(T\) cannot fully compensate for a coarse continuation path, while a large \(N\) cannot fully compensate for insufficient local refinement. The results show that the gain saturates beyond moderate budgets and can degrade when the budget is overly large. We therefore use \(N=T=200\) as a practical quality--cost trade-off.

\begin{table*}[!h]
\caption{\textbf{Sensitivity to sampling budget.}
We report quantitative results for different combinations of the continuation steps \(N\) and inner refinement steps \(T\) of Alg.~\ref{alg:main} on FFHQ motion deblurring.}
\centering
\small
\setlength{\tabcolsep}{5pt}
\renewcommand{\arraystretch}{1.15}
\resizebox{0.75\textwidth}{!}{%
\begin{tabular}{|l|ccc|c|}
\toprule
\((N,T)\)
& PSNR$\uparrow$ & SSIM$\uparrow$ & LPIPS$\downarrow$ & Runtime (s/image) \\
\midrule
\((50, 200)\)
& 35.82 & 0.924 & 0.069 & \(\sim 35\) \\

\((100, 200)\)
& 36.28 & 0.935 & 0.055 & \(\sim 65\) \\

\((300, 200)\)
& \textbf{37.08} & \textbf{0.953} & \textbf{0.039} & \(\sim 195\) \\

\((500, 200)\)
& 37.04 & 0.949 & 0.046 & \(\sim 280\) \\

\((200, 50)\)
& 35.81 & 0.943 & 0.047 & \(\sim 70\) \\

\((200, 100)\)
& 36.86 & 0.945 & 0.050 & \(\sim 110\) \\

\((200, 300)\)
& 36.95 & 0.950 & 0.041 & \(\sim 160\) \\

\((200, 500)\)
& 35.80 & 0.913 & 0.075 & \(\sim 215\) \\

\((300, 100)\)
& 35.34 & 0.938 & 0.053 & \(\sim 175\) \\

\((400, 50)\)
& 35.84 & 0.943 & 0.046 & \(\sim 200\) \\

\((500, 25)\)
& 36.40 & 0.950 & 0.040 & \(\sim 240\) \\

\((300, 300)\)
& 36.97 & 0.948 & 0.047 & \(\sim 225\) \\

\((400, 400)\)
& 36.80 & 0.940 & 0.054 & \(\sim 320\) \\

\((500, 500)\)
& 36.41 & 0.928 & 0.063 & \(\sim 425\) \\

Ours \((200, 200)\)
& 36.69 & 0.940 & 0.054 & \(\sim 130\) \\
\bottomrule
\end{tabular}%
}
\label{tab:nt_sensitivity}
\end{table*}

To further examine the low-runtime regime, we compare our method with DCDP under a similar computational budget in Tab.~\ref{tab:same_budget_motion_deblur}. With \(N=9\) and \(T=90\), our method runs in approximately \(2\) seconds per image and outperforms DCDP on FFHQ motion deblurring across all metrics. This suggests that the proposed continuation can be scaled down for efficiency while retaining competitive restoration quality.

\begin{table}[t]
\caption{\textbf{Low-runtime comparison under a similar computational budget.}
We compare different methods on FFHQ \(256\times256\) motion deblurring under a comparable runtime budget. Our method uses a reduced sampling budget with \(N=9,T=90\).}
\vspace{3mm}
\centering
\small
\setlength{\tabcolsep}{5pt}
\renewcommand{\arraystretch}{1.15}
\resizebox{0.7\columnwidth}{!}{%
\begin{tabular}{|l|ccc|c|}
\toprule
Method
& PSNR$\uparrow$ & SSIM$\uparrow$ & LPIPS$\downarrow$ & Runtime (s/img) \\
\midrule

DCDP
& 25.08 & 0.512 & 0.364 & \(\sim 2\) \\

Ours \((N=9, T=90)\)
& \textbf{27.31} & \textbf{0.603} & \textbf{0.264} & \(\sim 2\) \\
\bottomrule
\end{tabular}%
}
\label{tab:same_budget_motion_deblur}
\end{table}

\paragraph{Measurement noise level.}
We also evaluate different measurement noise levels \(\sigma_y\). This experiment is not intended as a method-parameter ablation, since \(\sigma_y\) specifies the observation noise in the inverse problem. Instead, it illustrates how restoration quality changes with measurement SNR. As shown in Tab.~\ref{tab:sigma_sensitivity}, performance decreases as \(\sigma_y\) increases, because fine-scale components become less identifiable under noisier observations. We use \(\sigma_y=0.05\) in the main experiments to follow the standard noisy inverse-problem setting and maintain a fair comparison with baselines.

\begin{table}[!h]
\caption{\textbf{Sensitivity to measurement noise level.}
We vary the observation noise level \(\sigma_y\) of Eq.~\eqref{inv_define} on FFHQ motion deblurring.}
\vspace{3mm}
\centering
\small
\setlength{\tabcolsep}{5pt}
\renewcommand{\arraystretch}{1.15}
\resizebox{0.5\columnwidth}{!}{%
\begin{tabular}{|l|ccc|}
\toprule
\(\sigma_y\)
& PSNR$\uparrow$ & SSIM$\uparrow$ & LPIPS$\downarrow$ \\
\midrule
0.00
& \textbf{38.53} & \textbf{0.971} & \textbf{0.034} \\

0.03
& 37.49 & 0.963 & 0.042 \\

0.05 (Ours)
& 36.69 & 0.940 & 0.054 \\

0.07
& 34.95 & 0.935 & 0.082 \\
\bottomrule
\end{tabular}%
}
\label{tab:sigma_sensitivity}
\end{table}

\paragraph{Other parameters.}
Finally, we vary representative continuation and fusion parameters one at a time while keeping all other parameters fixed. As shown in Tab.~\ref{tab:other_param_sensitivity}, reducing the low-frequency commitment weight, weakening the continuation schedule, or using an overly aggressive initial detail-unlocking value can degrade performance. In particular, increasing \(d_s\) significantly hurts restoration quality, indicating that premature detail adoption is harmful when high-frequency directions are still weakly identifiable. These results support the adopted default configuration.

\begin{table}[!h]
\caption{\textbf{Sensitivity to other continuation and fusion parameters.}
Each row changes one parameter while keeping all remaining parameters identical to the default configuration of Tab.~\ref{tab:param_config}. Results are evaluated on FFHQ motion deblurring.}
\vspace{3mm}
\centering
\small
\setlength{\tabcolsep}{4pt}
\renewcommand{\arraystretch}{1.15}
\resizebox{0.6\columnwidth}{!}{%
\begin{tabular}{|l|ccc|}
\toprule
Settings
& PSNR$\uparrow$ & SSIM$\uparrow$ & LPIPS$\downarrow$ \\
\midrule
\(w_{i,c}=0.8\)
& 35.21 & 0.942 & \textbf{0.045} \\

\(\lambda_i\) starts at \(0.30\) and decays to \(0\)
& 35.33 & \textbf{0.943} & \textbf{0.045} \\

\(d_s=0.3\)
& 30.78 & 0.813 & 0.187 \\

Ours
& \textbf{36.69} & 0.940 & 0.054 \\
\bottomrule
\end{tabular}%
}
\label{tab:other_param_sensitivity}
\end{table}

\subsubsection{Additional Qualitative Results}\label{add_qual}

Fig.~\ref{fig_wavelet_basis} provides qualitative comparisons of different wavelet bases, complementing the quantitative ablation in Tab.~\ref{tab:wavelet_freq_ablation}. We further provide additional qualitative results on FFHQ and ImageNet in Fig.~\ref{fig_app_2} and Fig.~\ref{fig_app_3}, covering motion deblurring, super-resolution, Gaussian deblurring, and random inpainting.

\begin{figure}[!h]
    \centering
    \begin{minipage}{0.49\linewidth}
        \centering
        \includegraphics[width=\linewidth]{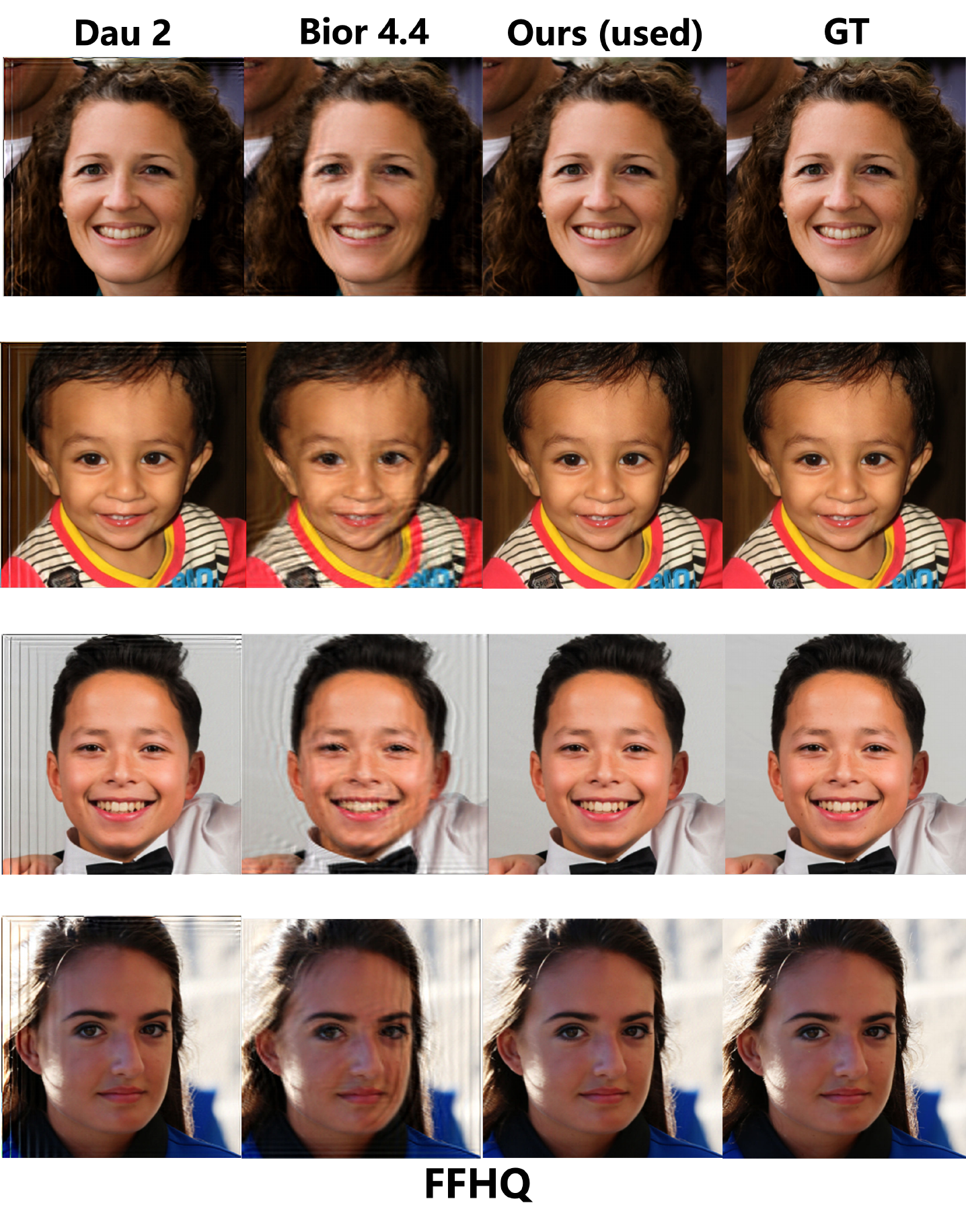}
    \end{minipage}
    \hfill
    \begin{minipage}{0.49\linewidth}
        \centering
        \includegraphics[width=\linewidth]{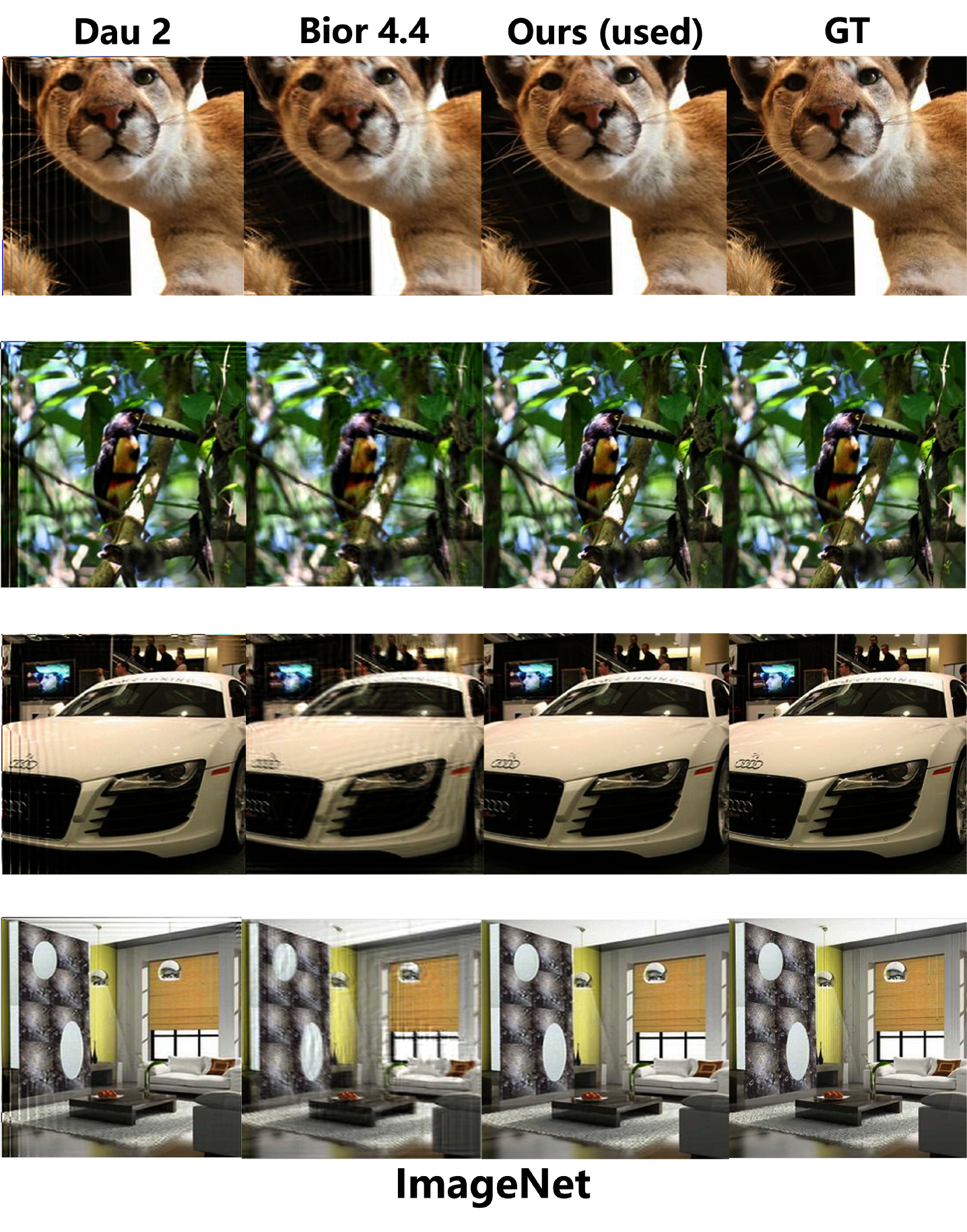}
    \end{minipage}
    \caption{\textbf{Qualitative comparison of different wavelet bases.}
    Results are evaluated on motion deblurring at \(256\times256\). Haar produces fewer artifacts and better preserves coarse structure and fine details among the tested bases.}
    \label{fig_wavelet_basis}
\end{figure}

\begin{figure*}[!h]
    \centering
    \includegraphics[width=\linewidth]{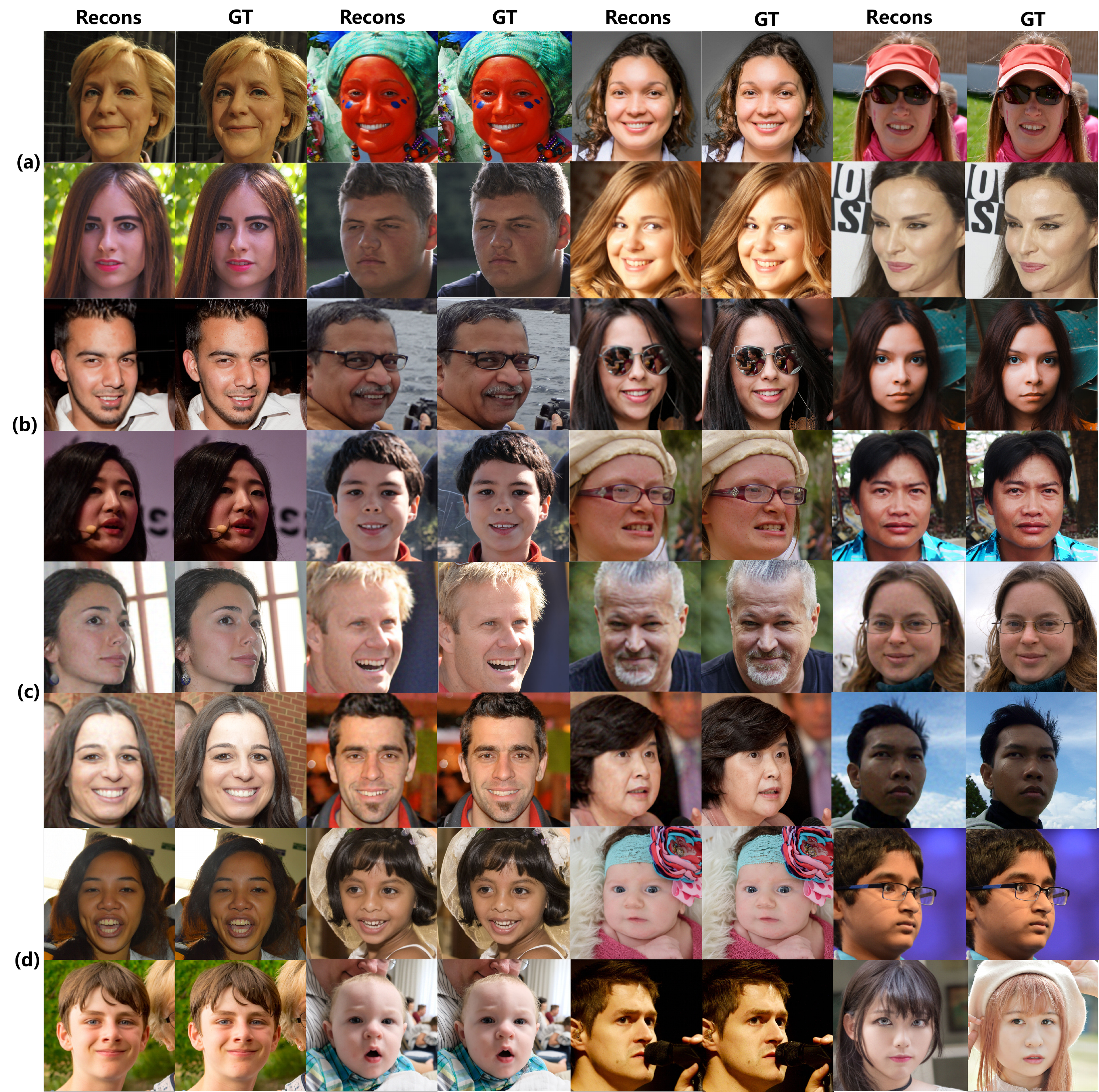}
    \vspace{-15pt}
    \caption{\textbf{Additional qualitative results on FFHQ \(256\times256\).}
    We show results for (a) motion deblurring, (b) super-resolution, (c) Gaussian deblurring, and (d) random inpainting. \textbf{Recons} denotes the reconstruction generated by our method.}
    \label{fig_app_2}
\end{figure*}

\begin{figure*}[!h]
    \centering
    \includegraphics[width=\linewidth]{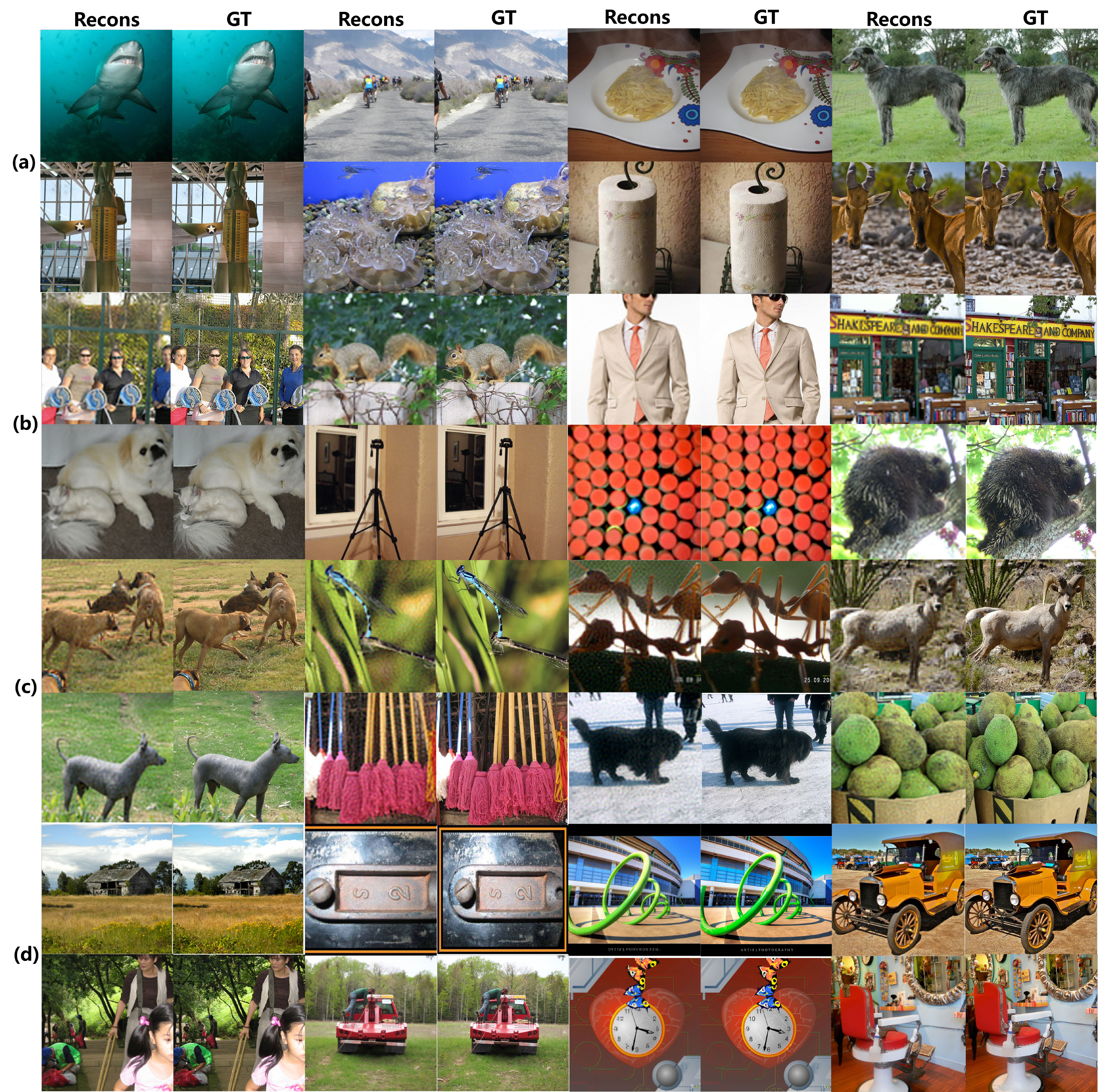}
    \vspace{-15pt}
    \caption{\textbf{Additional qualitative results on ImageNet \(256\times256\).}
    We show results for (a) motion deblurring, (b) super-resolution, (c) Gaussian deblurring, and (d) random inpainting. \textbf{Recons} denotes the reconstruction generated by our method.}
    \label{fig_app_3}
\end{figure*}

\subsection{Implementation Details}\label{imple}

\begin{table*}[t]
\caption{\textbf{Parameter configuration of our method.}
Default hyperparameters used in our main experiments.}
\vspace{3mm}
\centering
\small
\setlength{\tabcolsep}{5pt}
\renewcommand{\arraystretch}{1.15}
\resizebox{\textwidth}{!}{%
\begin{tabular}{|l|l|p{0.62\textwidth}|}
\toprule
Parameter & Default & Description \\
\midrule
\(N\)
& 200
& \texttt{num\_steps}, which determines the total number of diffusion iterations in line 3 of Alg.~\ref{alg:main}. \\

\(\sigma_{\max}\)
& 100
& Initial noise level of \(\sigma\) used in Alg.~\ref{alg:main}. \\

\(\sigma_{\min}\)
& 0.1
& Final noise level of \(\sigma\) used in Alg.~\ref{alg:main}. \\

\(\sigma_y\) in Eq.~\eqref{inv_define}
& 0.05
& Noise level for the measurement. \\

Timestep
& poly--7
& Time-step discretization scheme, polynomial with \(\rho=7\). \\

\(T\)
& 200
& Number of inner optimization steps in line 8 of Alg.~\ref{alg:main}. \\

Denoising Steps
& 5
& Number of denoising steps of the PF-ODE in line 4 of Alg.~\ref{alg:main}. \\

Cutoff-frequency schedule
& linear
& Coupled to the noise level, with start and end bounds set to \(0.4\) and \(1.0\), respectively. \\

\(\lambda_i\) in Eq.~\eqref{eq:fg_loss}
& cosine schedule
& Starts at \(0.35\) and decays to \(0\). \\

\(w_{i,c}\) in Eq.~\eqref{haar_fusion}
& 1.0
& Weight for the low-frequency approximation subband, where \(1.0\) means the refinement result is fully adopted. \\

\(w_{i,d}\) in Eq.~\eqref{haar_fusion}
& 0.2
& Base weight for the high-frequency subband. \\

\(d_s\) in Eq.~\eqref{eq:detail_gate}
& 0.05
& Initial value for detail unlocking. \\

\(d_e\) in Eq.~\eqref{eq:detail_gate}
& 1.0
& Final value for detail unlocking. \\
\bottomrule
\end{tabular}%
}
\label{tab:param_config}
\end{table*}

\paragraph{Diffusion Model.} We use the checkpoints based on EDM \cite{score3} provided by DPS \cite{dps} for all experiments, including ffhq\_10m.pt (357.1 MB) and imagenet256.pt (2.06 GB).

\paragraph{Hyper-parameters in Alg.~\ref{alg:main}.} To facilitate reproducibility, we provide the detailed parameter configurations of our method in Tab.~\ref{tab:param_config}. Unless otherwise stated, we use a single default hyperparameter configuration
across datasets and degradation types, without per-image or per-baseline tuning.
The default configuration is selected once and then
kept fixed for the main comparisons.

\paragraph{Configurations of Other Baselines:}\par\noindent

\noindent\textbf{DPS.}
We follow the official DPS configuration with a DDPM sampler with \(1000\) linear-schedule steps, \texttt{epsilon} prediction, learned-range variance, and clipping of the denoised estimate.

\noindent\textbf{DDRM.}
We follow the default configuration, setting \(\eta_B = 1.0\) and \(\eta = 0.85\), with 20 DDIM sampling steps.

\noindent\textbf{DDNM.}
We adopt the default configuration, setting \(\eta_B = 1.0\) and \(\eta = 0.85\), with 100 DDIM sampling steps.

\noindent\textbf{DCDP.}
We use the default setting described in the paper and directly employ the official open-source implementation for all reported results.

\noindent\textbf{FPS.}
We follow the default configuration in the paper, with \(M = 20\) and \(N = 1000\).

\noindent\textbf{DAPS.}
We follow the default configuration in the paper, with \(N = 200\) and 100 Langevin steps.

\noindent\textbf{FGPS.}
We follow the default configuration in the paper, with kernel size set to 61 and intensity set to 0.5 and 3.0 for motion blur and Gaussian blur, respectively.

\noindent\textbf{PSLD.}
We use the official implementation of PSLD with its default configurations. We use SD v1.5 for ImageNet as per its official default, providing a stronger backbone than the LDM-VQ4 used in other baselines.

\noindent\textbf{ReSample.}
All ReSample experiments are conducted based on the official implementation from the paper and open-source code, using a DDIM sampler with 500 steps.

\paragraph{Device.} All experiments on FFHQ \(256\times256\) are conducted on a single NVIDIA A6000
GPU. Experiments on ImageNet \(256\times256\) are performed on a single NVIDIA
A800/A100 GPU with 80 GB of memory, or on an RTX Pro 6000 Blackwell GPU with
96 GB of memory.

\subsection{Theorem, Lemma and Corollary}\label{proof}

We provide analysis under a standard linear-Gaussian measurement model
\begin{equation}
\boldsymbol{y}
=
\mathcal{A}\boldsymbol{x}^{\mathrm{gt}}
+
\boldsymbol{\varepsilon},
\qquad
\boldsymbol{\varepsilon}\sim \mathcal{N}(\boldsymbol{0},\sigma_y^2 \boldsymbol{I}),
\end{equation}
where \(\boldsymbol{x}^{\mathrm{gt}}\) denotes the ground-truth signal,
\(\mathcal{A}\) is a linear operator, and \(\sigma_y>0\) is the measurement noise level.
The full-band negative log-likelihood energy is
\begin{equation}
\mathcal{E}(\boldsymbol{x})
=
\frac{1}{2\sigma_y^2}
\|\mathcal{A}\boldsymbol{x}-\boldsymbol{y}\|_2^2,
\qquad
g(\boldsymbol{x})
:=
\nabla \mathcal{E}(\boldsymbol{x})
=
\frac{1}{\sigma_y^2}
\mathcal{A}^\top
(\mathcal{A}\boldsymbol{x}-\boldsymbol{y}).
\end{equation}
For notational simplicity, the first part of the analysis uses the measurement
noise scale \(\sigma_y\). In the intermediate posterior of Eq.~\eqref{had},
\(\beta_i\) plays the same role as an effective likelihood temperature, so the
same bounds apply after replacing \(\sigma_y\) by \(\beta_i\).

For a spectral cutoff radius \(\omega\), let \(P^F_{\omega}\) be the masked
spectrum operator in Eq.~\eqref{eq:PF_def}. We use \(\mathcal{P}_{\omega}\) to
denote the corresponding ideal residual-space orthogonal low-pass projector used
in the theoretical analysis, which satisfies \(\|\mathcal{P}_\omega\|_2\le 1\).
With the unitary Fourier normalization, \(\|P^F_\omega\boldsymbol{r}\|_2^2\)
equals the energy of the projected residual
\(\|\mathcal{P}_\omega\boldsymbol{r}\|_2^2\). Thus, \(P^F_\omega\) is used for
the implementation-oriented spectral notation in the main method, whereas
\(\mathcal{P}_\omega\) is used below as the equivalent residual-space
orthogonal projector.

Throughout this section, \(\sigma_i\) denotes the diffusion noise level at outer
step \(i\), \(\omega_i\) denotes the corresponding cutoff radius,
\(\lambda_i\in[0,1]\) denotes the continuation weight,
\(\hat{\boldsymbol{x}}_{0,i}\) denotes the PF-ODE anchor, and \(\beta_i\) is the
temperature coefficient in Eq.~\eqref{had}. We write
\[
\Delta_s(\boldsymbol{x},\sigma)
=
s_\theta(\boldsymbol{x},\sigma)
-
s^\star(\boldsymbol{x},\sigma)
\]
for the score-estimation error, where \(s^\star\) denotes the oracle score.

%========================================================
% Lemma 1: smoothness / conditioning
%========================================================
\begin{lemma}[Operator conditioning controls likelihood-gradient sensitivity]
\label{lem:cond_smooth}
The likelihood energy \(\mathcal{E}\) is \(L\)-smooth with
\begin{equation}
L=\frac{\|\mathcal{A}\|_2^2}{\sigma_y^2}.
\end{equation}
Moreover, for any \(\boldsymbol{x}_1,\boldsymbol{x}_2\),
\begin{equation}
\|g(\boldsymbol{x}_1)-g(\boldsymbol{x}_2)\|_2
=
\|\nabla \mathcal{E}(\boldsymbol{x}_1)-\nabla \mathcal{E}(\boldsymbol{x}_2)\|_2
\le
\frac{\|\mathcal{A}\|_2^2}{\sigma_y^2}
\|\boldsymbol{x}_1-\boldsymbol{x}_2\|_2.
\end{equation}
\end{lemma}

\begin{proof}
Since
\[
\mathcal{E}(\boldsymbol{x})
=
\frac{1}{2\sigma_y^2}
\|\mathcal{A}\boldsymbol{x}-\boldsymbol{y}\|_2^2,
\]
we have
\[
\nabla \mathcal{E}(\boldsymbol{x})
=
\frac{1}{\sigma_y^2}
\mathcal{A}^\top
(\mathcal{A}\boldsymbol{x}-\boldsymbol{y}),
\qquad
\nabla^2 \mathcal{E}(\boldsymbol{x})
=
\frac{1}{\sigma_y^2}
\mathcal{A}^\top \mathcal{A}.
\]
Hence
\[
\|\nabla^2\mathcal{E}(\boldsymbol{x})\|_2
=
\frac{1}{\sigma_y^2}
\|\mathcal{A}^\top \mathcal{A}\|_2
=
\frac{\|\mathcal{A}\|_2^2}{\sigma_y^2},
\]
so \(\mathcal{E}\) is \(L\)-smooth. Also,
\[
\nabla \mathcal{E}(\boldsymbol{x}_1)
-
\nabla \mathcal{E}(\boldsymbol{x}_2)
=
\frac{1}{\sigma_y^2}
\mathcal{A}^\top \mathcal{A}
(\boldsymbol{x}_1-\boldsymbol{x}_2),
\]
and taking operator norms yields the Lipschitz bound.
\end{proof}

To formalize the high-noise regime, consider a Tweedie-type estimator
\begin{equation}
\hat{\boldsymbol{x}}_0(\boldsymbol{x},\sigma)
=
\boldsymbol{x}
+
\sigma^2 s_\theta(\boldsymbol{x},\sigma),
\qquad
\boldsymbol{x}
=
\boldsymbol{x}^{\mathrm{gt}}
+
\sigma\boldsymbol{\epsilon},
\quad
\boldsymbol{\epsilon}\sim\mathcal{N}(\boldsymbol{0},\boldsymbol{I}),
\end{equation}
and let
\[
s^\star(\boldsymbol{x},\sigma)
=
\nabla_{\boldsymbol{x}}\log p_\sigma(\boldsymbol{x})
\]
denote the oracle score.

\begin{lemma}[Large-\(\sigma\) amplifies score errors into clean-estimate errors]
\label{lem:tweedie_amp}
For any \((\boldsymbol{x},\sigma)\),
\begin{equation}
\left\|
\hat{\boldsymbol{x}}_0(\boldsymbol{x},\sigma)
-
\boldsymbol{x}_0^\star(\boldsymbol{x},\sigma)
\right\|_2
=
\sigma^2
\left\|
s_\theta(\boldsymbol{x},\sigma)
-
s^\star(\boldsymbol{x},\sigma)
\right\|_2,
\end{equation}
where
\[
\boldsymbol{x}_0^\star(\boldsymbol{x},\sigma)
\triangleq
\boldsymbol{x}
+
\sigma^2 s^\star(\boldsymbol{x},\sigma)
\]
is the oracle Tweedie estimator. In particular, if
\[
\left\|
s_\theta(\boldsymbol{x},\sigma)
-
s^\star(\boldsymbol{x},\sigma)
\right\|_2
\le
\epsilon(\sigma),
\]
then
\begin{equation}
\left\|
\hat{\boldsymbol{x}}_0(\boldsymbol{x},\sigma)
-
\boldsymbol{x}_0^\star(\boldsymbol{x},\sigma)
\right\|_2
\le
\sigma^2 \epsilon(\sigma).
\end{equation}
\end{lemma}

\begin{proof}
By definition,
\[
\hat{\boldsymbol{x}}_0
-
\boldsymbol{x}_0^\star
=
\left(
\boldsymbol{x}+\sigma^2 s_\theta
\right)
-
\left(
\boldsymbol{x}+\sigma^2 s^\star
\right)
=
\sigma^2(s_\theta-s^\star).
\]
Taking norms yields the claim.
\end{proof}

\begin{theorem}[Gradient-mismatch sensitivity upper bound]
\label{gradient_mismatch}
Under the linear-Gaussian likelihood above, for any \((\boldsymbol{x},\sigma)\),
\begin{equation}
\left\|
g(\hat{\boldsymbol{x}}_0(\boldsymbol{x},\sigma))
-
g(\boldsymbol{x}_0^\star(\boldsymbol{x},\sigma))
\right\|_2
\le
\frac{\|\mathcal{A}\|_2^2}{\sigma_y^2}
\left\|
\hat{\boldsymbol{x}}_0(\boldsymbol{x},\sigma)
-
\boldsymbol{x}_0^\star(\boldsymbol{x},\sigma)
\right\|_2.
\end{equation}
If additionally
\[
\left\|
s_\theta(\boldsymbol{x},\sigma)
-
s^\star(\boldsymbol{x},\sigma)
\right\|_2
\le
\epsilon(\sigma),
\]
then
\begin{equation}
\left\|
g(\hat{\boldsymbol{x}}_0(\boldsymbol{x},\sigma))
-
g(\boldsymbol{x}_0^\star(\boldsymbol{x},\sigma))
\right\|_2
\le
\frac{\|\mathcal{A}\|_2^2}{\sigma_y^2}
\sigma^2
\epsilon(\sigma).
\end{equation}
\end{theorem}

\begin{proof}
The first inequality is Lemma~\ref{lem:cond_smooth} with
\(\boldsymbol{x}_1=\hat{\boldsymbol{x}}_0(\boldsymbol{x},\sigma)\) and
\(\boldsymbol{x}_2=\boldsymbol{x}_0^\star(\boldsymbol{x},\sigma)\).
The second follows from Lemma~\ref{lem:tweedie_amp}.
\end{proof}

\paragraph{Exact spectral form and conditional amplification.}
The preceding result is an upper-bound statement. We next give an exact spectral
characterization in the linear-Gaussian setting, which clarifies when the
\(\sigma^2\) factor yields actual amplification on measured directions rather
than merely a worst-case sensitivity envelope.

\begin{theorem}[Exact gradient-mismatch scaling]
\label{exact}
Consider the linear-Gaussian likelihood
\[
\mathcal{E}(\boldsymbol{x})
=
\frac{1}{2\sigma_y^2}
\|\mathcal{A}\boldsymbol{x}-\boldsymbol{y}\|_2^2,
\qquad
g(\boldsymbol{x})
=
\nabla \mathcal{E}(\boldsymbol{x})
=
\frac{1}{\sigma_y^2}
\mathcal{A}^\top
(\mathcal{A}\boldsymbol{x}-\boldsymbol{y}).
\]
Let
\[
\hat{\boldsymbol{x}}_0(\boldsymbol{x},\sigma)
=
\boldsymbol{x}
+
\sigma^2 s_\theta(\boldsymbol{x},\sigma),
\qquad
\boldsymbol{x}_0^\star(\boldsymbol{x},\sigma)
=
\boldsymbol{x}
+
\sigma^2 s^\star(\boldsymbol{x},\sigma),
\]
and define the score error
\[
\Delta_s(\boldsymbol{x},\sigma)
=
s_\theta(\boldsymbol{x},\sigma)
-
s^\star(\boldsymbol{x},\sigma).
\]
Then the likelihood-gradient mismatch satisfies the exact identity
\begin{equation}
g(\hat{\boldsymbol{x}}_0(\boldsymbol{x},\sigma))
-
g(\boldsymbol{x}_0^\star(\boldsymbol{x},\sigma))
=
\frac{\sigma^2}{\sigma_y^2}
\mathcal{A}^\top\mathcal{A}
\Delta_s(\boldsymbol{x},\sigma).
\end{equation}
If
\[
\mathcal{A}
=
U\operatorname{diag}(a_1,\ldots,a_r)V^\top
\]
is a singular-value decomposition, then
\begin{equation}
\left\|
g(\hat{\boldsymbol{x}}_0(\boldsymbol{x},\sigma))
-
g(\boldsymbol{x}_0^\star(\boldsymbol{x},\sigma))
\right\|_2^2
=
\frac{\sigma^4}{\sigma_y^4}
\sum_{j=1}^{r}
a_j^4
\left|
\left\langle
\boldsymbol{v}_j,
\Delta_s(\boldsymbol{x},\sigma)
\right\rangle
\right|^2.
\end{equation}
Consequently, for any measured subspace
\[
\mathcal{S}
=
\operatorname{span}\{\boldsymbol{v}_j:j\in J\},
\]
define the projected mismatch
\[
e_{\mathcal{S}}(\boldsymbol{x},\sigma)
:=
\Pi_{\mathcal{S}}
\left(
g(\hat{\boldsymbol{x}}_0(\boldsymbol{x},\sigma))
-
g(\boldsymbol{x}_0^\star(\boldsymbol{x},\sigma))
\right).
\]
Then
\begin{equation}
\|e_{\mathcal{S}}(\boldsymbol{x},\sigma)\|_2^2
=
\frac{\sigma^4}{\sigma_y^4}
\sum_{j\in J}
a_j^4
\left|
\left\langle
\boldsymbol{v}_j,
\Delta_s(\boldsymbol{x},\sigma)
\right\rangle
\right|^2.
\end{equation}
Thus, the mismatch component on \(\mathcal{S}\) scales quadratically with
\(\sigma\), modulated by the score-error energy on the same measured directions.
In particular, for \(\sigma_h>\sigma_l\), consider two states
\((\boldsymbol{x}_h,\sigma_h)\) and \((\boldsymbol{x}_l,\sigma_l)\). If
\begin{equation}
\left(
\sum_{j\in J}
a_j^4
\left|
\left\langle
\boldsymbol{v}_j,
\Delta_s(\boldsymbol{x}_h,\sigma_h)
\right\rangle
\right|^2
\right)^{1/2}
\ge
\left(\frac{\sigma_l}{\sigma_h}\right)^2
\left(
\sum_{j\in J}
a_j^4
\left|
\left\langle
\boldsymbol{v}_j,
\Delta_s(\boldsymbol{x}_l,\sigma_l)
\right\rangle
\right|^2
\right)^{1/2},
\end{equation}
then
\begin{equation}
\|e_{\mathcal{S}}(\boldsymbol{x}_h,\sigma_h)\|_2
\ge
\|e_{\mathcal{S}}(\boldsymbol{x}_l,\sigma_l)\|_2 .
\end{equation}
\end{theorem}

\begin{proof}
By the Tweedie form,
\[
\hat{\boldsymbol{x}}_0(\boldsymbol{x},\sigma)
-
\boldsymbol{x}_0^\star(\boldsymbol{x},\sigma)
=
\sigma^2
\Delta_s(\boldsymbol{x},\sigma).
\]
Since
\[
g(\boldsymbol{x})
=
\frac{1}{\sigma_y^2}
\mathcal{A}^\top
(\mathcal{A}\boldsymbol{x}-\boldsymbol{y}),
\]
we have
\[
g(\hat{\boldsymbol{x}}_0)
-
g(\boldsymbol{x}_0^\star)
=
\frac{1}{\sigma_y^2}
\mathcal{A}^\top\mathcal{A}
(
\hat{\boldsymbol{x}}_0-\boldsymbol{x}_0^\star
)
=
\frac{\sigma^2}{\sigma_y^2}
\mathcal{A}^\top\mathcal{A}
\Delta_s.
\]
Using
\[
\mathcal{A}^\top\mathcal{A}
=
V\operatorname{diag}(a_1^2,\ldots,a_r^2)V^\top
\]
gives
\[
\left\|
\mathcal{A}^\top\mathcal{A}\Delta_s
\right\|_2^2
=
\sum_{j=1}^{r}
a_j^4
\left|
\left\langle
\boldsymbol{v}_j,
\Delta_s
\right\rangle
\right|^2,
\]
which proves the exact spectral identity. Projecting onto
\(\mathcal{S}=\operatorname{span}\{\boldsymbol{v}_j:j\in J\}\) restricts the
same sum to \(j\in J\). The monotonic comparison follows directly from the
displayed projected identity and the stated condition.
\end{proof}

\begin{corollary}[Frequency truncation removes high-frequency mismatch]
\label{truncation}
Assume \(\mathcal{A}\) and the spectral mask \(\mathcal{P}_{\omega}\) are
simultaneously diagonalizable in the Fourier basis, with transfer function
\(a(\xi)\). Let
\[
e_{\mathrm{full}}(\sigma)
=
\frac{\sigma^2}{\sigma_y^2}
\mathcal{A}^\top\mathcal{A}
\Delta_s(\sigma),
\]
and
\[
e_{\omega}(\sigma)
=
\frac{\sigma^2}{\sigma_y^2}
\mathcal{A}^\top
\mathcal{P}_{\omega}^{\top}
\mathcal{P}_{\omega}
\mathcal{A}
\Delta_s(\sigma).
\]
Then
\begin{equation}
\|e_{\mathrm{full}}(\sigma)\|_2^2
-
\|e_{\omega}(\sigma)\|_2^2
=
\frac{\sigma^4}{\sigma_y^4}
\sum_{\xi\notin\Omega_\omega}
|a(\xi)|^4
|\widehat{\Delta_s}(\xi,\sigma)|^2
\ge 0.
\end{equation}
The inequality is strict whenever the score error has nonzero energy on
frequencies outside the exposed passband and the operator has nonzero transfer
there.
\end{corollary}

\begin{proof}
Under the simultaneous diagonalization assumption, \(\mathcal{A}\) acts in the
Fourier domain by multiplication with \(a(\xi)\), and
\(\mathcal{P}_\omega\) acts by the binary mask \(m_\omega(\xi)\). Hence
\[
\widehat{e_{\mathrm{full}}}(\xi,\sigma)
=
\frac{\sigma^2}{\sigma_y^2}
|a(\xi)|^2
\widehat{\Delta_s}(\xi,\sigma),
\]
whereas
\[
\widehat{e_{\omega}}(\xi,\sigma)
=
\frac{\sigma^2}{\sigma_y^2}
m_\omega(\xi)
|a(\xi)|^2
\widehat{\Delta_s}(\xi,\sigma).
\]
Since \(m_\omega(\xi)=1\) for \(\xi\in\Omega_\omega\) and
\(m_\omega(\xi)=0\) otherwise, subtracting the squared norms gives the stated
identity.
\end{proof}

\begin{theorem}[Inexact likelihood gradients can break descent]
\label{thm:unstable}
Let \(\mathcal{E}\) be \(L\)-smooth. Consider an inexact gradient step
\begin{equation}
\boldsymbol{x}^{+}
=
\boldsymbol{x}
-
\alpha
\left(
\nabla \mathcal{E}(\boldsymbol{x})+\boldsymbol{e}
\right),
\end{equation}
where \(\boldsymbol{e}\) is an additive gradient error and \(\alpha>0\). Then
\begin{equation}
\begin{aligned}
\mathcal{E}\left(\boldsymbol{x}^{+}\right)
\le\,
&
\mathcal{E}(\boldsymbol{x})
-
\alpha
\left(
1-\frac{L \alpha}{2}
\right)
\|\nabla \mathcal{E}(\boldsymbol{x})\|_2^2 \\
&+
\alpha(1+L\alpha)
\left|
\left\langle
\nabla \mathcal{E}(\boldsymbol{x}),
\boldsymbol{e}
\right\rangle
\right|
+
\frac{L \alpha^2}{2}
\|\boldsymbol{e}\|_2^2 .
\end{aligned}
\end{equation}
In particular, even if \(\alpha \leq 1/L\), at any non-stationary point
\(\nabla \mathcal{E}(\boldsymbol{x})\neq \boldsymbol{0}\), there exist errors
\(\boldsymbol{e}\) such that
\(\mathcal{E}(\boldsymbol{x}^{+})>\mathcal{E}(\boldsymbol{x})\).
Moreover, a sufficient condition for guaranteed decrease is
\begin{equation}
\alpha \le \frac{1}{L}
\qquad\text{and}\qquad
\|\boldsymbol{e}\|_2
\le
(\sqrt{5}-2)
\|\nabla \mathcal{E}(\boldsymbol{x})\|_2 .
\end{equation}
\end{theorem}

\begin{proof}
By \(L\)-smoothness, for any direction \(\boldsymbol{d}\),
\[
\mathcal{E}(\boldsymbol{x}+\boldsymbol{d})
\le
\mathcal{E}(\boldsymbol{x})
+
\left\langle
\nabla \mathcal{E}(\boldsymbol{x}),
\boldsymbol{d}
\right\rangle
+
\frac{L}{2}
\|\boldsymbol{d}\|_2^2 .
\]
Set
\[
\boldsymbol{d}
=
-\alpha
\left(
\nabla \mathcal{E}(\boldsymbol{x})+\boldsymbol{e}
\right).
\]
Then
\[
\begin{aligned}
\mathcal{E}(\boldsymbol{x}^{+})
\le\,
&
\mathcal{E}(\boldsymbol{x})
-
\alpha
\|\nabla \mathcal{E}(\boldsymbol{x})\|_2^2
-
\alpha
\left\langle
\nabla \mathcal{E}(\boldsymbol{x}),
\boldsymbol{e}
\right\rangle \\
&+
\frac{L\alpha^2}{2}
\left\|
\nabla \mathcal{E}(\boldsymbol{x})+\boldsymbol{e}
\right\|_2^2 .
\end{aligned}
\]
Expanding the squared norm gives
\[
\begin{aligned}
\mathcal{E}(\boldsymbol{x}^{+})
\le\,
&
\mathcal{E}(\boldsymbol{x})
-
\alpha
\left(
1-\frac{L\alpha}{2}
\right)
\|\nabla\mathcal{E}(\boldsymbol{x})\|_2^2 \\
&-
\alpha(1-L\alpha)
\left\langle
\nabla\mathcal{E}(\boldsymbol{x}),
\boldsymbol{e}
\right\rangle
+
\frac{L\alpha^2}{2}
\|\boldsymbol{e}\|_2^2 .
\end{aligned}
\]
Since
\[
-\alpha(1-L\alpha)
\left\langle
\nabla\mathcal{E}(\boldsymbol{x}),
\boldsymbol{e}
\right\rangle
\le
\alpha(1+L\alpha)
\left|
\left\langle
\nabla\mathcal{E}(\boldsymbol{x}),
\boldsymbol{e}
\right\rangle
\right|,
\]
we obtain the stated bound.

To show that descent can fail, assume
\(\nabla \mathcal{E}(\boldsymbol{x})\neq \boldsymbol{0}\). For any sufficiently
small \(t>0\), differentiability implies
\[
\mathcal{E}
\left(
\boldsymbol{x}
+
t\nabla \mathcal{E}(\boldsymbol{x})
\right)
>
\mathcal{E}(\boldsymbol{x}).
\]
Choose
\[
\boldsymbol{e}
=
-
\left(
1+\frac{t}{\alpha}
\right)
\nabla \mathcal{E}(\boldsymbol{x}).
\]
Then
\[
\boldsymbol{x}^{+}
=
\boldsymbol{x}
-
\alpha
\left(
\nabla \mathcal{E}(\boldsymbol{x})+\boldsymbol{e}
\right)
=
\boldsymbol{x}
+
t\nabla \mathcal{E}(\boldsymbol{x}),
\]
so \(\mathcal{E}(\boldsymbol{x}^{+})>\mathcal{E}(\boldsymbol{x})\).

For the sufficient decrease condition, assume \(\alpha\le 1/L\). Then
\[
1-\frac{L\alpha}{2}\ge \frac12,
\qquad
1+L\alpha\le 2,
\qquad
\frac{L\alpha^2}{2}\le \frac{\alpha}{2}.
\]
Using
\[
\left|
\left\langle
\nabla \mathcal{E}(\boldsymbol{x}),
\boldsymbol{e}
\right\rangle
\right|
\le
\|\nabla \mathcal{E}(\boldsymbol{x})\|_2
\|\boldsymbol{e}\|_2,
\]
we obtain
\[
\mathcal{E}(\boldsymbol{x}^{+})-\mathcal{E}(\boldsymbol{x})
\le
\alpha
\|\nabla \mathcal{E}(\boldsymbol{x})\|_2^2
\left(
-\frac12
+
2\rho
+
\frac{\rho^2}{2}
\right),
\qquad
\rho
:=
\frac{\|\boldsymbol{e}\|_2}
{\|\nabla \mathcal{E}(\boldsymbol{x})\|_2}.
\]
The quadratic
\[
-\frac12+2\rho+\frac{\rho^2}{2}\le 0
\]
holds for \(\rho\in[0,\sqrt{5}-2]\). Therefore
\[
\|\boldsymbol{e}\|_2
\le
(\sqrt{5}-2)
\|\nabla \mathcal{E}(\boldsymbol{x})\|_2
\]
guarantees
\[
\mathcal{E}(\boldsymbol{x}^{+})
\le
\mathcal{E}(\boldsymbol{x}).
\]
\end{proof}

\begin{corollary}[Band-limited likelihood reduces curvature and sensitivity]
\label{sensitive}
Define a filtered measurement energy
\begin{equation}
\mathcal{E}_{\omega}(\boldsymbol{x})
\triangleq
\frac{1}{2\sigma_y^2}
\left\|
\mathcal{P}_{\omega}
(
\mathcal{A}\boldsymbol{x}-\boldsymbol{y}
)
\right\|_2^2,
\end{equation}
where \(\mathcal{P}_{\omega}\) is a linear operator with
\(\|\mathcal{P}_{\omega}\|_2\le 1\), such as the ideal orthogonal low-pass
projector associated with Eq.~\eqref{eq:PF_def}. Then
\(\mathcal{E}_{\omega}\) is \(L_{\omega}\)-smooth with
\begin{equation}
L_{\omega}
=
\frac{1}{\sigma_y^2}
\left\|
\mathcal{A}^{\top}
\mathcal{P}_{\omega}^{\top}
\mathcal{P}_{\omega}
\mathcal{A}
\right\|_2
\le
\frac{\|\mathcal{A}\|_2^2}{\sigma_y^2}
\|\mathcal{P}_{\omega}\|_2^2
\le
\frac{\|\mathcal{A}\|_2^2}{\sigma_y^2}
=
L.
\end{equation}
Consequently, the gradient Lipschitz constant and the smoothness-driven
step-size restriction are reduced under the band-limited likelihood. This helps
explain why full-band enforcement at high noise levels can be more schedule- and
operator-sensitive.
\end{corollary}

\begin{corollary}[Mixed frequency exposure directly attenuates likelihood-gradient error]
\label{cor:mixed_gradient_error}
Let the frequency-exposure objective at outer step \(i\) be
\begin{equation}
\mathcal{E}_i(\boldsymbol{x})
=
\frac{1}{2\sigma_y^2}
\left[
(1-\lambda_i)
\|\mathcal{A}\boldsymbol{x}-\boldsymbol{y}\|_2^2
+
\lambda_i
\left\|
\mathcal{P}_{\omega_i}
(
\mathcal{A}\boldsymbol{x}-\boldsymbol{y}
)
\right\|_2^2
\right],
\end{equation}
where \(\mathcal{P}_{\omega_i}\) is the orthogonal frequency projector onto the
exposed passband and \(\lambda_i\in[0,1]\). Define
\begin{equation}
Q_i
:=
(1-\lambda_i)\boldsymbol{I}
+
\lambda_i
\mathcal{P}_{\omega_i}^{\top}
\mathcal{P}_{\omega_i}.
\end{equation}
Assume that the clean-estimate error induced by the score approximation satisfies
\begin{equation}
\hat{\boldsymbol{x}}_0
-
\boldsymbol{x}_0^\star
=
\sigma^2
\Delta_s(\boldsymbol{x},\sigma),
\end{equation}
where \(\boldsymbol{x}_0^\star\) denotes the oracle clean estimate and
\(\Delta_s(\boldsymbol{x},\sigma)\) denotes the score-estimation error. Then the
likelihood-gradient mismatch satisfies
\begin{equation}
\boldsymbol{e}_i(\sigma)
:=
\nabla \mathcal{E}_i(\hat{\boldsymbol{x}}_0)
-
\nabla \mathcal{E}_i(\boldsymbol{x}_0^\star)
=
\frac{\sigma^2}{\sigma_y^2}
\mathcal{A}^{\top}
Q_i
\mathcal{A}
\Delta_s(\boldsymbol{x},\sigma).
\end{equation}
Moreover, if \(\mathcal{A}\) and \(\mathcal{P}_{\omega_i}\) are simultaneously
diagonalized in the Fourier basis, with transfer function \(a(\xi)\) and binary
mask \(m_{\omega_i}(\xi)\in\{0,1\}\), then
\begin{equation}
\|\boldsymbol{e}_i(\sigma)\|_2^2
=
\frac{\sigma^4}{\sigma_y^4}
\sum_{\xi}
q_i(\xi)^2
|a(\xi)|^4
|\widehat{\Delta_s}(\xi,\sigma)|^2,
\end{equation}
where
\begin{equation}
q_i(\xi)
=
(1-\lambda_i)
+
\lambda_i m_{\omega_i}(\xi).
\end{equation}
Consequently, compared with the full-band likelihood-gradient mismatch
\begin{equation}
\|\boldsymbol{e}_{\mathrm{full}}(\sigma)\|_2^2
=
\frac{\sigma^4}{\sigma_y^4}
\sum_{\xi}
|a(\xi)|^4
|\widehat{\Delta_s}(\xi,\sigma)|^2,
\end{equation}
the mixed frequency-exposure objective satisfies
\begin{equation}
\|\boldsymbol{e}_{\mathrm{full}}(\sigma)\|_2^2
-
\|\boldsymbol{e}_i(\sigma)\|_2^2
=
\frac{\sigma^4}{\sigma_y^4}
\sum_{\xi}
\left(1-q_i(\xi)^2\right)
|a(\xi)|^4
|\widehat{\Delta_s}(\xi,\sigma)|^2
\ge 0.
\end{equation}
For an ideal binary low-pass mask, this reduction can be written as
\begin{equation}
\|\boldsymbol{e}_{\mathrm{full}}(\sigma)\|_2^2
-
\|\boldsymbol{e}_i(\sigma)\|_2^2
=
\frac{\sigma^4}{\sigma_y^4}
\sum_{\xi\notin\Omega_{\omega_i}}
\left[1-(1-\lambda_i)^2\right]
|a(\xi)|^4
|\widehat{\Delta_s}(\xi,\sigma)|^2
\ge 0.
\end{equation}
Thus, frequency exposure directly attenuates the likelihood-gradient error by
suppressing the contribution of unexposed frequency bands. This effect is
distinct from, and complementary to, the reduction of the effective curvature and
the resulting relaxation of smoothness-driven step-size constraints.
\end{corollary}

\begin{proof}
The gradient of \(\mathcal{E}_i\) is
\begin{equation}
\nabla \mathcal{E}_i(\boldsymbol{x})
=
\frac{1}{\sigma_y^2}
\mathcal{A}^{\top}
\left[
(1-\lambda_i)\boldsymbol{I}
+
\lambda_i
\mathcal{P}_{\omega_i}^{\top}
\mathcal{P}_{\omega_i}
\right]
(
\mathcal{A}\boldsymbol{x}-\boldsymbol{y}
)
=
\frac{1}{\sigma_y^2}
\mathcal{A}^{\top}
Q_i
(
\mathcal{A}\boldsymbol{x}-\boldsymbol{y}
).
\end{equation}
Therefore,
\begin{align}
\boldsymbol{e}_i(\sigma)
&=
\nabla \mathcal{E}_i(\hat{\boldsymbol{x}}_0)
-
\nabla \mathcal{E}_i(\boldsymbol{x}_0^\star)
\\
&=
\frac{1}{\sigma_y^2}
\mathcal{A}^{\top}
Q_i
\mathcal{A}
(
\hat{\boldsymbol{x}}_0-\boldsymbol{x}_0^\star
).
\end{align}
Using
\[
\hat{\boldsymbol{x}}_0-\boldsymbol{x}_0^\star
=
\sigma^2
\Delta_s(\boldsymbol{x},\sigma)
\]
gives
\begin{equation}
\boldsymbol{e}_i(\sigma)
=
\frac{\sigma^2}{\sigma_y^2}
\mathcal{A}^{\top}
Q_i
\mathcal{A}
\Delta_s(\boldsymbol{x},\sigma).
\end{equation}

When \(\mathcal{A}\) and \(\mathcal{P}_{\omega_i}\) are diagonalized in the
Fourier basis, \(\mathcal{A}\) acts by multiplication with \(a(\xi)\), and
\(Q_i\) acts by multiplication with
\[
q_i(\xi)
=
(1-\lambda_i)
+
\lambda_i m_{\omega_i}(\xi).
\]
Hence the Fourier coefficient of \(\boldsymbol{e}_i(\sigma)\) at frequency
\(\xi\) is
\[
\widehat{\boldsymbol{e}_i}(\xi,\sigma)
=
\frac{\sigma^2}{\sigma_y^2}
q_i(\xi)
|a(\xi)|^2
\widehat{\Delta_s}(\xi,\sigma),
\]
which implies
\[
\|\boldsymbol{e}_i(\sigma)\|_2^2
=
\frac{\sigma^4}{\sigma_y^4}
\sum_{\xi}
q_i(\xi)^2
|a(\xi)|^4
|\widehat{\Delta_s}(\xi,\sigma)|^2.
\]
The full-band case corresponds to \(q_i(\xi)=1\) for all \(\xi\). Subtracting
the mixed-objective error from the full-band error yields
\[
\|\boldsymbol{e}_{\mathrm{full}}(\sigma)\|_2^2
-
\|\boldsymbol{e}_i(\sigma)\|_2^2
=
\frac{\sigma^4}{\sigma_y^4}
\sum_{\xi}
\left(1-q_i(\xi)^2\right)
|a(\xi)|^4
|\widehat{\Delta_s}(\xi,\sigma)|^2.
\]
Since \(\lambda_i\in[0,1]\) and \(m_{\omega_i}(\xi)\in\{0,1\}\), we have
\(q_i(\xi)\in[0,1]\), and therefore \(1-q_i(\xi)^2\ge 0\).

For an ideal binary low-pass mask, \(m_{\omega_i}(\xi)=1\) for
\(\xi\in\Omega_{\omega_i}\) and \(m_{\omega_i}(\xi)=0\) for
\(\xi\notin\Omega_{\omega_i}\). Thus \(q_i(\xi)=1\) inside the exposed passband
and \(q_i(\xi)=1-\lambda_i\) outside it. The reduction therefore vanishes inside
\(\Omega_{\omega_i}\) and equals \(1-(1-\lambda_i)^2\) outside the passband.
\end{proof}

\begin{lemma}[Non-expansiveness of spectral masking]
\label{lem:spectral_nonexp}
Let \(P^F_{\omega}\) be defined in Eq.~\eqref{eq:PF_def} with the unitary
normalization of \(\mathcal{F}\). Then for any residual \(\boldsymbol{r}\),
\begin{equation}
\|P^F_{\omega}(\boldsymbol{r})\|_2
\le
\|\boldsymbol{r}\|_2.
\end{equation}
\end{lemma}

\begin{proof}
The frequency shift is a permutation and hence norm-preserving. With the chosen
normalization, \(\mathcal{F}\) is unitary, and Parseval's identity gives
\[
\|\mathcal{F}(\boldsymbol{r})\|_2
=
\|\boldsymbol{r}\|_2.
\]
Multiplication by \(M_{\omega}\) is an orthogonal projection in the Fourier
domain and is therefore non-expansive.
\end{proof}

\begin{theorem}[Low-pass-only likelihood induces non-identifiability]
\label{thm:nonident}
Fix a cutoff radius \(\omega\) and define the low-pass loss
\begin{equation}
\mathcal{L}_{\mathrm{freq}}(\boldsymbol{x})
=
\left\|
P^F_{\omega}
(
\mathcal{A}\boldsymbol{x}-\boldsymbol{y}
)
\right\|_2^2.
\end{equation}
If there exists \(\boldsymbol{d}\neq \boldsymbol{0}\) such that
\[
P^F_{\omega}(\mathcal{A}\boldsymbol{d})=\boldsymbol{0},
\]
then
\begin{equation}
\mathcal{L}_{\mathrm{freq}}(\boldsymbol{x}+\boldsymbol{d})
=
\mathcal{L}_{\mathrm{freq}}(\boldsymbol{x})
\qquad
\text{for all }\boldsymbol{x}.
\end{equation}
Hence the low-pass data term alone cannot distinguish \(\boldsymbol{x}\) along
the subspace \(\ker(P^F_{\omega}\mathcal{A})\).
\end{theorem}

\begin{proof}
By linearity of \(P^F_{\omega}\),
\[
\mathcal{L}_{\mathrm{freq}}(\boldsymbol{x}+\boldsymbol{d})
=
\left\|
P^F_{\omega}
(
\mathcal{A}\boldsymbol{x}-\boldsymbol{y}
)
+
P^F_{\omega}
(
\mathcal{A}\boldsymbol{d}
)
\right\|_2^2.
\]
If \(P^F_{\omega}(\mathcal{A}\boldsymbol{d})=\boldsymbol{0}\), the value is
unchanged.
\end{proof}

%========================================================
% Lemma: Haar detail insensitivity
%========================================================
\begin{lemma}[Low-pass data term is insensitive to Haar details under idealization]
\label{lp_insensitive_detail}
Let \(W\) denote an orthonormal Haar transform and let
\[
\boldsymbol{z}=W\boldsymbol{x}
\]
be its coefficients, with
\[
\boldsymbol{z}
=
(\boldsymbol{z}_{\mathrm{c}},\boldsymbol{z}_{\mathrm{d}})
\]
denoting coarse and detail blocks. Let \(\mathcal{P}_{\omega_i}\) denote the
ideal low-pass orthogonal projector induced by the spectral mask in
Sec.~\ref{log_density} with cutoff radius \(\omega_i\). Assume that the
composition
\[
\mathcal{P}_{\omega_i}\circ\mathcal{A}\circ W^\top
\]
annihilates Haar-detail components in the idealized early regime, i.e.,
\begin{equation}
\mathcal{P}_{\omega_i}
\mathcal{A}
W^\top
(\boldsymbol{0},\boldsymbol{z}_{\mathrm{d}})
=
\boldsymbol{0}
\qquad
\text{for all } \boldsymbol{z}_{\mathrm{d}}.
\end{equation}
Then the low-pass residual depends only on \(\boldsymbol{z}_{\mathrm{c}}\), and
\begin{equation}
\nabla_{\boldsymbol{z}_{\mathrm{d}}}
\left\|
\mathcal{P}_{\omega_i}
\left(
\mathcal{A}(W^\top \boldsymbol{z})
-
\boldsymbol{y}
\right)
\right\|_2^2
=
\boldsymbol{0}.
\end{equation}
When the annihilation condition holds approximately, the gradient along the
Haar-detail coordinates is correspondingly small.
\end{lemma}

\begin{proof}
By linearity,
\[
W^\top(\boldsymbol{z}_{\mathrm{c}},\boldsymbol{z}_{\mathrm{d}})
=
W^\top(\boldsymbol{z}_{\mathrm{c}},\boldsymbol{0})
+
W^\top(\boldsymbol{0},\boldsymbol{z}_{\mathrm{d}}).
\]
The annihilation assumption gives
\[
\mathcal{P}_{\omega_i}
\mathcal{A}
W^\top
(\boldsymbol{0},\boldsymbol{z}_{\mathrm{d}})
=
\boldsymbol{0}.
\]
Therefore,
\[
\mathcal{P}_{\omega_i}
\left(
\mathcal{A}
W^\top
(\boldsymbol{z}_{\mathrm{c}},\boldsymbol{z}_{\mathrm{d}})
-
\boldsymbol{y}
\right)
=
\mathcal{P}_{\omega_i}
\left(
\mathcal{A}
W^\top
(\boldsymbol{z}_{\mathrm{c}},\boldsymbol{0})
-
\boldsymbol{y}
\right),
\]
which is independent of \(\boldsymbol{z}_{\mathrm{d}}\). Hence the squared
low-pass residual is constant along the Haar-detail coordinates, and its
gradient with respect to \(\boldsymbol{z}_{\mathrm{d}}\) is zero.
\end{proof}

%========================================================
% Corollary: prior-dominated details
%========================================================
\begin{corollary}[Detail coefficients remain prior-dominated in the idealized low-pass regime]
\label{detail_prior_dominated}
This corollary analyzes the low-pass component of the intermediate posterior and
is intended to explain why detail directions are weakly constrained early in the
trajectory. Consider the idealized low-pass component of the intermediate target
at outer step \(i\),
\begin{equation}
\pi_i(\boldsymbol{x}_0)
\propto
\exp\left(
-
\frac{1}{2\beta_i^2}
\left\|
\mathcal{P}_{\omega_i}
\left(
\mathcal{A}(\boldsymbol{x}_0)-\boldsymbol{y}
\right)
\right\|_2^2
\right)
\cdot
\mathcal{N}
\left(
\boldsymbol{x}_0;
\hat{\boldsymbol{x}}_{0,i},
\sigma_i^2\boldsymbol{I}
\right).
\end{equation}
Let
\[
\boldsymbol{z}=W\boldsymbol{x}_0,
\qquad
\hat{\boldsymbol{z}}_i
=
W\hat{\boldsymbol{x}}_{0,i}.
\]
Under the conditions of Lemma~\ref{lp_insensitive_detail}, the conditional
distribution of detail coefficients is approximately
\begin{equation}
\boldsymbol{z}_{\mathrm{d}}
\mid
(\boldsymbol{y},\boldsymbol{z}_{\mathrm{c}})
\approx
\mathcal{N}
\left(
\hat{\boldsymbol{z}}_{\mathrm{d},i},
\sigma_i^2\boldsymbol{I}
\right).
\end{equation}
In particular, when \(\sigma_i\) is large, the posterior variance of
\(\boldsymbol{z}_{\mathrm{d}}\) remains large, so Langevin refinement can
introduce high-variance and operator-dependent spurious details in the detail
bands.
\end{corollary}

\begin{proof}
Since \(W\) is orthonormal, the anchor Gaussian remains isotropic in Haar
coordinates:
\[
\mathcal{N}
\left(
\boldsymbol{x}_0;
\hat{\boldsymbol{x}}_{0,i},
\sigma_i^2\boldsymbol{I}
\right)
=
\mathcal{N}
\left(
\boldsymbol{z};
\hat{\boldsymbol{z}}_i,
\sigma_i^2\boldsymbol{I}
\right).
\]
Lemma~\ref{lp_insensitive_detail} implies that the low-pass term depends only on
\(\boldsymbol{z}_{\mathrm{c}}\) and is constant with respect to
\(\boldsymbol{z}_{\mathrm{d}}\) in the idealized low-pass regime. Therefore the
joint density factorizes into a term depending on \(\boldsymbol{z}_{\mathrm{c}}\)
and an independent Gaussian term in \(\boldsymbol{z}_{\mathrm{d}}\). This gives
\[
\boldsymbol{z}_{\mathrm{d}}
\mid
(\boldsymbol{y},\boldsymbol{z}_{\mathrm{c}})
\approx
\mathcal{N}
\left(
\hat{\boldsymbol{z}}_{\mathrm{d},i},
\sigma_i^2\boldsymbol{I}
\right).
\]
\end{proof}

\end{document}